%% file: paper.tex
\documentclass[10pt,a4paper]{amsart} 
\usepackage[T1]{fontenc}
\usepackage{amsmath,amsfonts,amssymb,graphicx,url,amsthm, mathrsfs}
\numberwithin{equation}{section}
\usepackage[left=2.5cm,right=2.5cm,top=2.5cm,bottom=2.5cm]{geometry}
\usepackage{enumitem,kantlipsum}
\usepackage{xcolor,color,nicefrac,svg}
\usepackage{graphicx, bbding}
\usepackage{booktabs}
\usepackage{csquotes} 
\usepackage{caption}
\usepackage{subcaption}
\usepackage{mwe}
\usepackage{afterpage}
\usepackage[noend]{algpseudocode}
\usepackage[section,ruled]{algorithm}
\usepackage[export]{adjustbox}
\usepackage{natbib} 
\usepackage{setspace}
\usepackage[hyperfootnotes=true]{hyperref}
\hypersetup{colorlinks=true,raiselinks=true,breaklinks=true,citecolor=blue}
\usepackage{tikz,pgfplots,pgfplotstable}
\pgfplotsset{compat=newest}
\usetikzlibrary{arrows,decorations,backgrounds,positioning,calc,matrix}
\usepgfplotslibrary{groupplots}
\usetikzlibrary{external}
\tikzset{
    png export/.style={
        external/system call/.add={}{; convert -density 600 -transparent white "\image.pdf" "\image.png"},
        /pgf/images/external info,
        /pgf/images/include external/.code={%
            \includegraphics
            [width=\pgfexternalwidth,height=\pgfexternalheight]
            {##1.png}%
        },
    },
    png export,
}

\newtheorem{theorem}{Theorem}[section]
\newtheorem{lemma}[theorem]{Lemma}
\newtheorem{corollary}[theorem]{Corollary}

\newtheorem{definition}[theorem]{Definition}

\newtheorem{remark}[theorem]{Remark}

\usepackage{xparse}
\usepackage{verbatim}
\newsavebox{\fminipagebox}
\NewDocumentEnvironment{fminipage}{m O{\fboxsep}}
 {\par\kern#2\noindent\begin{lrbox}{\fminipagebox}
  \begin{minipage}{#1}\ignorespaces}
 {\end{minipage}\end{lrbox}%
  \makebox[#1]{%
    \kern\dimexpr-\fboxsep-\fboxrule\relax
    \fbox{\usebox{\fminipagebox}}%
    \kern\dimexpr-\fboxsep-\fboxrule\relax
  }\par\kern#2
 }

\def\letters{a,b,c,d,e,f,g,h,i,j,k,l,m,n,o,p,q,r,s,t,u,v,w,x,y,z}
\def\Letters{A,B,C,D,E,F,G,H,I,J,K,L,M,N,O,P,Q,R,S,T,U,V,W,X,Y,Z}
\makeatletter
\@for \@l:=\Letters \do{%
  \expandafter\edef\csname\@l bb\endcsname{\noexpand\ensuremath{%
  \noexpand\mathbb{\@l}}}%
  \expandafter\edef\csname\@l bf\endcsname{{\noexpand\bf \@l}}%
  \expandafter\edef\csname\@l cal\endcsname{\noexpand\ensuremath{%
  \noexpand\mathcal{\@l}}}%
  \expandafter\edef\csname\@l eu\endcsname{\noexpand\ensuremath{%
  \noexpand\EuScript{\@l}}}%
  \expandafter\edef\csname\@l frak\endcsname{\noexpand\ensuremath{%
  \noexpand\mathfrak{\@l}}}%
  \expandafter\edef\csname\@l rm\endcsname{{\noexpand\rm \@l}}%
  \expandafter\edef\csname\@l scr\endcsname{\noexpand\ensuremath{%
  \noexpand\mathscr{\@l}}}%
}
\@for \@l:=\letters \do{%
  \expandafter\edef\csname\@l bf\endcsname{{\noexpand\bf \@l}}%
  \expandafter\edef\csname\@l frak\endcsname{\noexpand\ensuremath{%
  \noexpand\mathfrak{\@l}}}%
  \expandafter\edef\csname\@l scr\endcsname{\noexpand\ensuremath{%
  \noexpand\mathscr{\@l}}}%
}
\makeatother
\definecolor{shadecolor}{rgb}{0.6, 0.6, 0.6} 
\definecolor{red}{rgb}{1,0,0.2} 
\definecolor{darkgreen}{rgb}{0, 0.6, 0}

\newcommand{\isdef}{\mathrel{\mathrel{\mathop:}=}}
\newcommand{\defis}{\mathrel{=\mathrel{\mathop:}}}

\newcommand{\R}{\mathbb{R}}
\newcommand{\N}{\mathbb{N}}

\newcommand{\kernel}{\mathcal{K}}

\renewcommand{\d}{\operatorname{d}\!}
\newcommand{\bs}{\boldsymbol}

\DeclareMathOperator{\sig}{sign}
\DeclareMathOperator{\spn}{span}
\DeclareMathOperator{\rank}{rank}

\DeclareMathOperator{\trace}{tr}

\DeclareMathOperator*\argmin{\arg\!\min}

\DeclareMathOperator\supp{supp}

\newcommand{\diag}{\operatorname{diag}}
\newcommand{\StateIndent}{\hspace{\algorithmicindent}}
\begin{document}
\title{Fast empirical scenarios}
 \thanks{We gratefully acknowledge support from the SNF grant 100018\_189086 ``Scenarios''.}
\author{Michael Multerer}
\address{
Michael Multerer,
Euler Institute,
USI Lugano,
Via la Santa 1, 6962 Lugano, Svizzera.}
\email{michael.multerer@usi.ch}
\author{Paul Schneider}
\address{
Paul Schneider,
USI Lugano and SFI,
Via Buffi 6, 6900 Lugano, Svizzera.}
\email{paul.schneider@usi.ch}
\author{Rohan Sen}
\address{
Rohan Sen,
Euler Institute,
USI Lugano,
Via la Santa 1, 6962 Lugano, Svizzera.}
\email{rohan.sen@usi.ch}

\maketitle
\input{abstract}
\input{revised_contents_2}%

\bibliographystyle{plainnat}
\bibliography{literature}
\end{document}

%% file: abstract.tex
\begin{abstract}

We seek to extract a small number of
representative scenarios from large panel
data that are consistent with sample moments. Among two novel algorithms,
the first identifies scenarios that have not been observed before, and comes
with a scenario-based representation of covariance
matrices. The second proposal selects important data points from states
of the world that have already realized, and are consistent with
higher-order sample moment information. Both algorithms are efficient
to compute and lend themselves to consistent scenario-based modeling
and multi-dimensional numerical integration that can be used for interpretable decision-making under uncertainty.
Extensive numerical benchmarking studies and an application in portfolio
optimization favor the proposed algorithms.
\end{abstract}

%% file: revised_contents_2.tex
\section{Introduction}\label{sec:intro}

Multi-dimensional data in various fields require efficient processing for 
informed decision-making. In many practical applications, the moments 
of the sample realizations induced by the sample distribution, together with the sample realizations themselves,  become central 
to inference tasks. For example, investors
care not only about the variance of outcomes but also about the higher-order moments of the distribution of outcomes, in particular, about the possibility of extremely adverse outcomes. Similarly, in the case of estimating the stochastic discount factor (SDF) for asset pricing, one must ensure that it incorporates adequate information about the higher moments of the return distribution as it is important for modeling tail risk, see \cite{Schneider2012,Almeida2017, Almeida2022}. As a result, portfolio optimization involving risky assets 
necessitates covariance matrices or even higher-order and/or non-linear functions 
of the moments, while the corresponding risk management is often scenario-based. 
The factor structure in interest rates also suggests a scenario-based approach, 
see \cite{ENGLE2017333}. 
More generally, the topic of summarizing information by replacing large samples 
of data with a small number of carefully weighted scenarios has found traction 
in many different communities, eg.\,\,scalable Bayesian statistics, see
\cite{Huggins2016CoresetsFS}, clustering and optimization, see \cite{feldman}, etc. In the discipline of explainable artificial intelligence, the scenarios' probabilities can be used to define observation-specific explanations that give rise to a novel class of surrogate models, see \cite{ghidini2024}.

In this article, we are concerned with finding scenarios that exploit the 
availability of realizations from large samples of multi-dimensional data sets. 
The theoretical basis of our problem lies within the
\emph{truncated moment problem} (TMP), that underlies the theory of 
multivariate quadrature integration rules, see 
\cite{Laurent2009-ee, Lasserre,schmuedgen17}. The TMP asks the
question whether a finite sequence can be identified as the moment sequence of an
underlying non-negative Borel measure, and if so, how to find another representing
discrete measure (having finite support) with the smallest number of atoms that 
generates the same moment sequence. 
Keeping in mind data-driven applications, we introduce the concept of the
\emph{empirical moment problem} (EMP), in analogy to the TMP, in which moments are 
derived from the sample measure induced by the data. 
In particular, the EMP can be realized as a quadrature problem that aims to reduce 
the support set of the sample measure by choosing representative scenarios with the 
additional constraints of non-negativity and normalization of the corresponding weights.

\subsection{Related work} 
Characterizing finite sequences as the moments of a non-negative Borel measure is the 
main idea behind the TMP and we refer to \cite{Laurent2009-ee, Lasserre, schmuedgen17} 
for a comprehensive discussion of the same. In particular, given a sequence of moments
of a non-negative Borel measure, an algorithm is provided in \cite{Lasserre}, that 
seeks to extract the finite support of another representing measure that gives rise 
to the same sequence. In its prototypical form, it is akin to the multivariate Prony's 
method, see \cite{Prony}. However, for practical purposes, especially for the EMP, 
Lasserre's algorithm quickly becomes numerically and computationally demanding as 
the number of dimensions grows. Moreover, the number of extracted atoms is typically high. 
A linear programming approach is proposed by \cite{Gauss_LP}, which is also found to be 
numerically prohibitive for dimensions greater than three.

The EMP can also be formulated as a numerical quadrature problem, wherein the goal is to 
approximate
\begin{equation}\label{eq:numerical_quadrature}
    \int_{\Omega} h(\bs x) \d\mu(\bs x) \approx \sum_{j=1}^m w_j h(\bs \xi_j),
\end{equation}
for a given probability measure $\mu$ on $\Omega \subset \R^d$ by carefully choosing the 
\emph{quadrature nodes} $\bs \xi_1,\ldots,\bs \xi_m$ and the corresponding \emph{weights} 
$w_1,\ldots,w_m$. Joint optimization of both the nodes and the weights, in general, 
corresponds to a non-convex problem, see \cite{Jen_PhD}. Existing approaches take Monte 
Carlo samples of the nodes (if the underlying distribution is known and can be simulated 
from), and then find the optimal weights, that minimize a root mean squared quadrature 
error. In the quasi-Monte Carlo (QMC) literature, the nodes are chosen by deterministic, 
low-discrepancy sequences, that loosely translate to a \enquote{well-spread} set of nodes, 
and then choosing uniform weights, see 
\cite{SOMMARIVA20091324, BosSDeMarchiSommarivaVianello, bittante, BSV10}.

With the assumption  that $h$ belongs to some \emph{reproducing kernel Hilbert space} 
(RKHS) $\Hcal$ of functions with reproducing kernel $\kernel$, the worst-case quadrature 
error may be expressed using the reproducing property as, see \cite{Jen_PhD, Bach2017},
\begin{equation}\label{eq:quadrature_error}
\begin{split}
    \sup_{\|h\|_{\Hcal}\leq 1}\bigg| \int_{\Omega} h(\bs x)\d\mu(\bs x) 
    - \sum_{j=1}^m w_j h(\bs \xi_j)\bigg| &= \sup_{\|h\|_{\Hcal}\leq 1} 
    \bigg|\bigg\langle h, \int_{\Omega}\kernel(\bs x, \cdot)\d\mu(\bs x) 
    - \sum_{j=1}^m w_j \kernel(\bs \xi_j, \cdot)\bigg\rangle_{\Hcal}\bigg|\\
    &\leq \bigg\| \int_{\Omega}\kernel(\bs x, \cdot)\d\mu(\bs x) 
    - \sum_{j=1}^m w_j \kernel(\bs \xi_j,\cdot)\bigg\|_{\Hcal}
\end{split}
\end{equation}
An approximation is obtained by choosing ideal candidates for the nodes and weights that 
minimize the right-hand side of the inequality \eqref{eq:quadrature_error}. We note that 
the literature on kernel-based quadrature rules is extensive, and we list a few that are 
relevant to our problem. \cite{Jen_PhD} finds optimal weights given Monte Carlo nodes. 
\cite{Bach2017, Belhadji2019KernelQW, hayakawa2022positively} choose nodes based on random 
designs, which are based on either a Mercer-based decomposition, see \cite{Bach2017}, or 
a Nystr{\"o}m approximation of the kernel.

Scattered data approximation, see \cite{Wendland_2004, Fasshauer_2007}, is another field 
which is related to the EMP. The question is how to construct kernel interpolants that can 
approximate functions in the RKHS sufficiently well. Our proposal is closest in spirit to 
this literature, see \cite{DeMarchi2010, PS2011}. \cite{RandomizedDiscreteMeasures} and \cite{GHTP} also 
provide algorithms that perform greedy subsampling of nodes from a set 
of samples. 

\subsection{Contributions}
We propose two algorithms for scenario extraction, where the first one can be applied to 
both the TMP and the EMP, while the second one is suited solely for the EMP. The first algorithm  produces scenarios with uniform weighting utilizing Householder reflections, specializing the TMP to the uniform measure. It is the fastest algorithm considered in 
this paper and reveals a set of uniformly distributed \emph{covariance scenarios} whose moment 
sequence matches up perfectly with that of the sample measure up to second-order moments.

The second algorithm is designed for solving the EMP, wherein the goal is to choose a good 
representative set of $m$ scenarios from a given set of $N$ samples, that matches the 
polynomial moments of the data samples sufficiently well in the regime of $m\ll N$. 
In particular, we reduce the support set of the sample measure to produce a finite atomic 
probability measure that ensures both positivity and normalization of the atomic weights. 
There are several reasons why such convex specifications of the weights are preferable: 
the RKHS may be mis-specified if the weights are negative and also to maintain positivity 
of the integral operator $h \mapsto \int h(\bs x) \d\mu(\bs x)$, 
see \cite{hayakawa2022positively}. Furthermore, the normalization also helps us in assessing 
the relative importance of the scenarios in describing the sample moment information and can 
thus be used for generative modeling of the underlying data.
The proposed algorithm offers a computational solution to the data-dependent orthogonal 
matching pursuit in the RKHS as in \cite{PS2011}, and is hence referred to as 
\emph{orthogonal matching pursuit} (OMP). It is based on the pivoted-Cholesky, see, e.g.,
\cite{HPS12}, decomposition of the kernel matrix which exempts us from choosing the scenarios 
(or atoms of the reduced measure) using a Mercer expansion of the associated kernel of our 
proposed RKHS.

We demonstrate the robustness, computational efficiency, and adaptivity of the algorithm 
for large samples of multivariate data and contrast it with existing approaches in 
the literature. Furthermore, we demonstrate the capacity of the scenarios (extracted 
with only sample moment information) to capture tail risk as well. We show this in the 
context of portfolio optimization of panel data of asset returns with (non-smooth) 
expected shortfall constraint.


Finally, we prove that basis pursuit or LASSO, see 
\cite{StatisticalLearning, foucart2013mathematical},
induced by $\ell_1$-regularized least squares, a standard approach in the literature on 
compressive sensing and machine learning, does not lend itself to recovering optimal
quadrature rules in the context of the EMP. To the best of our knowledge, we provide 
the first result that shows the inadequacy of the LASSO in constructing quadrature rules 
from sample measures. As a consequence of this, first-order proximal algorithms cannot 
be used in this context either.

\subsection{Outline} The remainder of this article is organized as follows. 
In Section~\ref{sec:TMP}, we formally introduce the TMP and the notion of moment matrices, 
based on which we propose covariance scenarios. In Section~\ref{sec:EMP}, we present the
empirical version, the EMP, and reformulate the problem in an RKHS framework. 
In Section~\ref{sec:scenarioextraction}, we propose the OMP algorithm, and demonstrate its 
applicability for the EMP. Moreover, we comment on
why basis pursuit cannot work in the context of 
the EMP. Section~\ref{sec:numExp} addresses numerical experiments, where we benchmark our 
proposed algorithms against existing approaches. In Section~\ref{sec:Conclusion},
we conclude and identify areas for future research. 

\section{The Truncated Moment Problem}\label{sec:TMP}
In this section, we review the key theoretical background of the 
TMP as the underlying mathematical foundation. Moreover, 
we make the connection of TMPs with quadrature rules in the multi-dimensional
framework. Finally, we present a specialized algorithm for the direct extraction 
of scenarios from covariance matrices in a fast manner.

\subsection{Background}\label{subsec:background} 
Let $\Omega \subset \R^d$ and $\Bscr$ denote the Borel $\sigma$-algebra on $\Omega$. 
Suppose that we are given a finite real sequence 
$\bs y = \big(y_{\bs \alpha}\big)_{|\bs \alpha| \leq q}$ 
indexed by $\bs \alpha \isdef (\alpha_1,\ldots,\alpha_d) \in \N^d$ with 
$|\bs \alpha| \isdef \alpha_1 + \cdots + \alpha_d  \leq q$, where $q \in \N$. 
We say that $\bs y$ has a \emph{representing measure} if there exists a
Borel measure \( \mu\colon\Bscr \to [0, \infty)\) such that 
\begin{equation}\label{eq:mom_prob}
y_{\boldsymbol\alpha}=\int_{ \Omega}{\boldsymbol x}^{\boldsymbol\alpha}\d\mu
\quad\text{for }|\boldsymbol\alpha|\leq q.
\end{equation}
If \eqref{eq:mom_prob} holds, $\bs y$ is called a \emph{truncated moment sequence} (TMS). 
The TMP asks: How to check if a finite sequence $\bs y$ has a representing measure? 
If such a measure exists, how do we find it?

\cite{bayerandteichmann06} proved a key result: \emph{``if a finite sequence 
$\bs y = \big(y_{\bs \alpha}\big)_{|\bs \alpha| \leq q}$ has a representing measure 
$\mu$, then it has another representing measure $\nu$ which has finite support with 
$\supp(\nu) \subseteq \supp(\mu)$''}. For truncated moment problems, the question of 
existence and the subsequent recovery of these finite representing measures require 
the notion of a \emph{moment matrix} and its associated \emph{flat extension}, 
see \cite{Laurent2009-ee,Lasserre,schmuedgen17}. Towards this end, we denote  by
\(\Pscr_q(\Omega)\) the space of all polynomials
of total degree \(q\), i.e.,
\(
\Pscr_q(\Omega)\isdef\spn\{{\boldsymbol x}^{\boldsymbol\alpha}:
\bs x \in \Omega, |\boldsymbol\alpha|\leq q\}
\), whose dimension is known to be
\begin{equation}\label{eq:mq}
m_q\isdef\binom{q+d}{d}.
\end{equation}
Letting the row vector
\begin{equation}\label{eq:monombasis}
  {\boldsymbol\tau_q}({\boldsymbol x})\isdef
\Big[1,x_1,\ldots,x_d,x_1^2,x_1x_2,\ldots x_d^2,\ldots,x_1^q,\ldots,x_d^q
\Big]  
\end{equation}
denote the monomial basis in \(\Pscr_q(\Omega)\), we may express every
\(p\in\Pscr_q(\Omega)\) according to
\[
p({\boldsymbol x})={\boldsymbol\tau_q}({\boldsymbol x}){\boldsymbol p}
\isdef\sum_{|\boldsymbol\alpha|\leq q}
p_{\boldsymbol\alpha}{\boldsymbol x}^{\boldsymbol\alpha}
\]
for a suitable coefficient vector
$\boldsymbol p=(p_{\boldsymbol\alpha})_{|\boldsymbol\alpha|\leq q}
\in\Rbb^{m_q}$. Every TMS $\bs y$ defines a linear functional $\Lscr_{\bs y}$ 
acting on $\Pscr_q(\Omega)$ as
\[
\Lscr_{\bs y}\Big(\sum_{|\bs \alpha| \leq q} p_{\bs \alpha}\bs x^{\bs \alpha}\Big) 
\isdef \sum_{|\bs \alpha| \leq q} p_{\bs \alpha} y_{\bs \alpha}.
\]
\begin{definition}\label{def:flatextension}  
Let \({\boldsymbol y}\in\Rbb^{m_{2q}}\), cp.\eqref{eq:mom_prob}. 
The \emph{moment matrix} of order $q$ is defined as
\begin{equation}\label{eq:moment_matrix}
    {\boldsymbol M}_{\boldsymbol y}
\isdef\mathscr{L}_{\boldsymbol y}({\boldsymbol\tau_q^\top}{\boldsymbol\tau_q})
=\big[y_{\boldsymbol\alpha+\boldsymbol\beta}\big]_{|\boldsymbol\alpha|,
|\boldsymbol\beta|\leq q}\in\Rbb^{m_q\times m_q},
\end{equation}
where the action of \(\mathscr{L}_{\boldsymbol y}\) on
\({\boldsymbol\tau_q^\top}{\boldsymbol\tau_q}\)
has to be understood element-wise, cp.\ \eqref{eq:monombasis}. If there exists
a vector \(\tilde{\boldsymbol y}\in\Rbb^{m_{2(q+1)}}\) such that
\(\tilde{y}_i=y_i\) for \(i=1,\ldots, m_{2q}\) and
${\boldsymbol M}_{\tilde{\boldsymbol y}}$ is positive semi-definite with
\(\rank {\boldsymbol M}_{\tilde{\boldsymbol y}}
=\rank{\boldsymbol M}_{{\boldsymbol y}}\),
we call \({\boldsymbol M}_{\tilde{\boldsymbol y}}\) a \emph{flat extension}
of \({\boldsymbol M}_{{\boldsymbol y}}\).
\end{definition}

An $r$-\emph{atomic measure} 
$\mu$ is a positive linear combination of $r$ Dirac measures i.e.,
\begin{equation}\label{eq:ratomicmeasure}
\mu = \sum_{j=1}^{r} \lambda_{j}\delta_{\boldsymbol\xi_j},\quad
\lambda_{1}, \ldots , \lambda_{r} >0. 
\end{equation}
The points ${\boldsymbol\xi}_{j}\in\Rbb^d$ are called the \emph{atoms}
of $\mu$, which we refer to as \emph{scenarios} in our context.
We state the following important result due to \cite{curtofialkow96} that 
characterizes TMS having finite representing measures, cp.\cite{Lasserre}.
\begin{theorem}[\citep{curtofialkow96}]\label{res:curtofialkow}
Let \(\bs y \in \R^{m_{2q}}\) with $\rank \bs M_{\bs y} = r$. Then $\bs y$ has a 
unique $r$-atomic representing measure on $\R^d$ iff 
the moment matrix $\bs M_{\bs y}$ is positive semi-definite and has a 
flat extension. 
\end{theorem}

A particular consequence of Theorem~\ref{res:curtofialkow} is that any moment 
matrix \({\boldsymbol M}_{\boldsymbol y}\in\Rbb^{m_q\times m_q}\) associated to 
an \(r\)-atomic measure \(\mu\) can be represented in the \emph{Vandermonde form}. 
\begin{equation}\label{eq:VandermondeForm}
    {\boldsymbol M}_{\boldsymbol y}= \sum_{j=1}^{r}\lambda_j \boldsymbol
    \tau_q^\top(\boldsymbol \xi_j)
\boldsymbol \tau_q(\boldsymbol \xi_j)= 
{\boldsymbol V}^\top_q({\boldsymbol\xi}_1,\ldots,
{\boldsymbol\xi}_r){\boldsymbol\Lambda}
{\boldsymbol V}_q({\boldsymbol\xi}_1,\ldots,{\boldsymbol\xi}_r),
\end{equation}
with ${\boldsymbol\xi}_1,\ldots,{\boldsymbol\xi}_r$  the scenarios and
their probabilities \({\boldsymbol\Lambda}
\isdef\operatorname{diag}(\lambda_1,\ldots,\lambda_r)\),
see \cite{Lasserre,schmuedgen17}. In \eqref{eq:VandermondeForm}, 
the matrix \({\boldsymbol V}_q({\boldsymbol\xi}_1,\ldots,{\boldsymbol\xi}_r)\)
is the (generalized) Vandermonde matrix
\begin{equation}\label{eq:Vdmde}
{\boldsymbol V_q}({\boldsymbol\xi}_1,\ldots,{\boldsymbol\xi}_r)
\isdef\begin{bmatrix}{\boldsymbol\tau_q}({\boldsymbol\xi}_1)\\ 
\vdots\\
{\boldsymbol\tau_q}({\boldsymbol\xi}_r)
\end{bmatrix}\in\Rbb^{r\times m_q}.
\end{equation}

\subsection{Connection to multivariate quadrature}
For a Borel measure $\mu$ with support $\Omega \subset \R^d$ and 
$\Pscr_q(\Omega) \subset L^1(\Omega, \mu)$, 
a \emph{quadrature rule of degree} $q$ and 
\emph{size} $m\in\Nbb$ consists of nodes 
$\bs \xi_1,\ldots,\bs \xi_m$ and positive weights 
$\lambda_1,\ldots,\lambda_m$ such that 
\begin{equation}\label{eq:quadrature}
    \int_{\Omega} p(\bs x) \d\mu = \sum_{j=1}^m \lambda_j  p(\bs \xi_j) \quad
    \text{for all } p \in \Pscr_q(\Omega). 
\end{equation}
The next result makes explicit the equivalence between the existence of finite 
atomic representing measures for TMPs to the existence of quadrature rules 
in multi-dimensions.
\begin{theorem}[\citep{bayerandteichmann06}]\label{res:gen_tchakaloff}
    Let $\mu$ be a Borel measure with support set $\Omega \subset \R^d$ 
    such that $\Pscr_{q}(\Omega) \subset L^1(\Omega, \mu)$. Then there exists 
    a quadrature rule of degree $q$ and size $1\leq m \leq m_q$.
\end{theorem}

Considering the truncated sequence that is given by the moments of $\mu$ up to 
degree $q$, the existence of a quadrature rule of degree $q$, cp.\ 
Theorem \ref{res:gen_tchakaloff}, implies an affirmative answer to the 
question of having another finite atomic representing measure for the same sequence. 
The latter measure is given by a linear combination of the Dirac measures supported
at the scenarios.

\begin{remark}\label{rem:GaussQuad_1}
    For a quadrature rule of degree $q$, the size estimate $m_q$ in 
    Theorem \ref{res:gen_tchakaloff} is not sharp in general. For certain sets and 
    measures, there exist Gaussian-type quadrature rules for which the size is
    considerably smaller than the one given by Theorem \ref{res:gen_tchakaloff}, 
    see \cite{Xu2020,fialkow1999, curtofialkow2002} and the references therein. Especially 
    in the case $d=1$, any Borel measure on $\R$ having moments up to degree $q$ admits a 
    Gaussian quadrature of degree $q$ with size $\leq \lfloor q/2 \rfloor + 1$.  
\end{remark}

Given a TMS, Lasserre's algorithm, see \cite[Algorithm 4.2]{Lasserre} relies on the 
Vandermonde form cp.\ \eqref{eq:VandermondeForm} and Theorem \ref{res:curtofialkow}. 
As a result, this method necessitates computing a flat extension  of the moment matrix. 
However, obtaining a flat extension is difficult, particularly in high dimensions and 
for $q>1$, see \cite{Helton2012}. 

\subsection{Covariance scenarios} 
In this paragraph, we suggest an approach 
for the particular case when 
${\boldsymbol M}_{\boldsymbol y}\in\Rbb^{m_1\times m_1},
{\boldsymbol y}\in\Rbb^{m_2}$. To compute the Vandermonde form as
in \eqref{eq:VandermondeForm} without the need to obtain a flat extension, 
we rely on Householder reflections, see\ \cite{Householder}, for the following theorem.
\begin{theorem}\label{thm:CovScen}
Let ${\boldsymbol R}\in\Rbb^{m_1\times r}$ be a matrix root of the
covariance matrix ${\boldsymbol M}_{\boldsymbol y}\in\Rbb^{m_1\times m_1}$,
that is \({\boldsymbol M_{\boldsymbol y}}={\boldsymbol R}{\boldsymbol R}^\top\).
Then, the Vandermonde form of ${\boldsymbol M}_{\boldsymbol y}$ reads
\(
{\boldsymbol M}_{\boldsymbol y}={\boldsymbol V}^\top
{\boldsymbol\Lambda}{\boldsymbol V}
\)
with ${\boldsymbol V}=\sqrt{r}{\boldsymbol H}_{\boldsymbol v}
{\boldsymbol R}^\top$ and \({\boldsymbol \Lambda}=\frac 1 r{\boldsymbol I}\),
where
\(
{\boldsymbol H}_{\boldsymbol v}\isdef{\boldsymbol I}-2\frac{{\boldsymbol v}
{\boldsymbol v}^\top}{{\boldsymbol v}^\top{\boldsymbol v}}
\)
for 
${\boldsymbol v}\isdef{\boldsymbol r}-\frac{1}{\sqrt{r}}{\boldsymbol 1}$.
Herein, ${\boldsymbol r}^\top$ is the first row of ${\boldsymbol R}$ and
\({\boldsymbol 1}\in\Rbb^r\) is the vector of all $1$'s.
\end{theorem}
\begin{proof}
There holds ${\boldsymbol H}_{\boldsymbol v}{\boldsymbol w}={\boldsymbol w}$
for \({\boldsymbol w}\perp{\boldsymbol v}\) and 
\({\boldsymbol H}_{\boldsymbol v}{\boldsymbol v}=-{\boldsymbol v}\). Hence,
choosing \({\boldsymbol v}={\boldsymbol r}-\lambda{\boldsymbol 1}\),
and \(\lambda\isdef \|{\boldsymbol r}\|_2/\|{\boldsymbol 1}\|_2
=\|{\boldsymbol r}\|_2/\sqrt{r}\),
one readily verifies
\({\boldsymbol H}_{\boldsymbol v}{\boldsymbol r}=\lambda{\boldsymbol 1}\).

Moreover, it is straightforward to see that \({\boldsymbol H}_{\boldsymbol v}\)
is orthogonal. Therefore, we arrive at the Vandermonde form
\(
{\boldsymbol M}_{\boldsymbol y}={\boldsymbol R}
{\boldsymbol H}_{\boldsymbol v}^\top{\boldsymbol H}_{\boldsymbol v}
{\boldsymbol R}^\top
={\boldsymbol V}^\top{\boldsymbol \Lambda}{\boldsymbol V}
\)
with \({\boldsymbol \Lambda}=\lambda^2{\boldsymbol I}
=\|{\boldsymbol r}\|_2^2/r{\boldsymbol I}\).
Finally, we obtain the assertion by
noticing that \(\|{\boldsymbol r}\|_2^2
={\boldsymbol r}^\top{\boldsymbol r}=({\boldsymbol M}_{\boldsymbol y})_{1,1}=1.\)
\end{proof}

Theorem \ref{thm:CovScen} demonstrates that it is possible to match up to second 
moments exactly using unique scenarios from any input moment matrix of degree one. 
The reason this yields unique scenarios without a
flat extension of the moment matrix lies in the uniform distribution of
the weights. The restriction to $q=1$ is due to preserving the Vandermonde form, 
which could otherwise not be guaranteed. 
We list the computational steps of Theorem~\ref{thm:CovScen} 
in Algorithm~\ref{algo:cov_scenarios}.
We refer to the unique scenarios obtained from this algorithm
as \emph{covariance scenarios}. It yields a new decomposition of moment matrices 
up to second order in terms of scenarios 
$\boldsymbol \xi _1,\ldots, \boldsymbol \xi _r$ computed according to 
\ref{algo:cov_scenarios} below, each realizing with probability $1/r$.

\begin{algorithm}[htb]\caption{Covariance scenarios}
\label{algo:cov_scenarios}\
\begin{flushleft}
\begin{tabular}{ll}
\textbf{input:}  & symmetric and positive semidefinite moment matrix
${\boldsymbol M_{\boldsymbol y}}\in\mathbb{R}^{m_1\times m_1}$, tolerance \(\varepsilon > 0\) \\
\textbf{output:} & scenarios \(\Xi \isdef \big\{\boldsymbol
\xi_1,\cdots,\boldsymbol\xi_r\big\} \subset \Rbb^d\)
\end{tabular}
\end{flushleft}
\begin{algorithmic}[1]
\smallskip
\State compute a matrix root
\(
    \boldsymbol M_{\boldsymbol y} = \boldsymbol R\boldsymbol R^\top
\)
\State set \(\boldsymbol r \isdef \boldsymbol R^\top 
\boldsymbol e_1,\ \lambda \isdef 1/\sqrt{r},\  \boldsymbol v \isdef \boldsymbol r 
-  \lambda\boldsymbol 1,\ \gamma \isdef 2/\|\boldsymbol v\|_2^2  \)
\State set \(\boldsymbol H_{\boldsymbol v} \isdef 
\boldsymbol I - \gamma\boldsymbol v \boldsymbol v^\top\ \text{and}\ 
{\boldsymbol V} \isdef \frac{1}{\lambda}\boldsymbol H_v \boldsymbol R^\top\)
\State set \(\boldsymbol \xi_j 
= \big[{\boldsymbol V}^\top\boldsymbol e_{j}\big]_{i=2}^{d+1}
\quad\text{for }j = 1,\ldots,r\)
\end{algorithmic}
\end{algorithm}

Note that the \emph{covariance scenarios} grant consistent use of moments and scenarios, 
which is important for certain applications, as noted already in the introduction. In finance, for instance,  portfolio 
optimization usually requires covariance matrices, while risk management is 
scenario-based. The construction introduced above guarantees that both function in 
lockstep. Furthermore, the extracted covariance scenarios would typically reflect 
future potential instances of the cross-section of asset returns and can be used for 
generative modeling of the risk landscape of the assets.


\section{The empirical moment problem}\label{sec:EMP}
In this section, we consider the case when the 
underlying probability measure is known empirically from i.i.d.\ draws, as is 
typically assumed in data-driven applications. We first formulate the problem as 
an empirical alternative to the truncated moment problem of Section \ref{sec:TMP}. 
We then propose a reformulation of the problem in an appropriate RKHS that can be 
used to embed the polynomial moments of the samples.

\subsection{Problem formulation} Let 
\( X = \{\boldsymbol x_{1},\ldots,\boldsymbol x_{N}\} \subset \Omega \subset \Rbb^d\),
denote the set of data samples. The associated empirical measure is given by
\[
    \widehat{\Pbb}\isdef\frac 1 N\sum_{i=1}^N\delta_{{\boldsymbol x}_i}.
\]
It satisfies \(\widehat{\Pbb}(X)=1\), and is hence a probability measure.
With the empirical measure at hand, it is straightforward to compute
the associated empirical truncated moment sequence 
\begin{equation}\label{eq:empirical_moment_seq}
\widehat{\boldsymbol y} 
= \big[\widehat{y}_{\boldsymbol \alpha}\big]_{|\boldsymbol \alpha| \leq 2q}
= \bigg[\int_{\Omega}\boldsymbol x^{\boldsymbol \alpha}
\d\widehat{\Pbb}\bigg]_{|\bs \alpha| \leq 2q} = \frac{1}{N} \sum^{N}_{i=1}\boldsymbol
\tau_{2q}^\top(\boldsymbol x_{i}) = \frac{1}{N}\boldsymbol
V^\top_{2q}(\boldsymbol x_1,\ldots,\boldsymbol x_N)
\boldsymbol 1 \in \Rbb^{m_{2q}}.
\end{equation}
With $\widehat{\bs y}$ admitting $\widehat \Pbb$ as a representing measure, we focus 
on performing moment-matching with respect to the empirical measure for accurate 
scenario-based representation of the data samples. In analogy to Section \ref{sec:TMP}, 
we thus pose the EMP: Does the empirical TMS $\widehat{\bs y}$ admit another representing 
probability measure $\Pbb^\star$ such that $\supp({\Pbb}^\star) \subset \supp(\widehat\Pbb)$? 
If such a measure exists, then how to obtain it?

Note that under the assumption of $\supp({\Pbb^\star}) \subset \supp(\widehat \Pbb)$, 
we seek an appropriate set of scenarios $\bs \Xi \isdef \big\{\boldsymbol\xi _1,\ldots,\boldsymbol
\xi _m\big\} \subset X$ with $m\ll N $. This corresponds to a compressed version of the sample 
measure, i.e.,
\begin{equation}\label{eq:compressed_measure}
    {\Pbb}^\star = \sum_{j=1}^{m}\lambda_j\delta_{\boldsymbol \xi_j}, 
    \quad \lambda_j\geq 0, \quad \sum_{j=1}^m\lambda_j=1,
\end{equation}
with the associated truncated moment sequence 
\begin{equation}\label{eq:compressed_moment_seq}
{\boldsymbol y^\star}
= \big[{y}^\star_{\boldsymbol \alpha}\big]_{|\boldsymbol \alpha| \leq 2q} 
= \bigg[\int_{\Omega}\boldsymbol x^{\boldsymbol \alpha}
\d{\Pbb}^\star\bigg]_{|\boldsymbol \alpha| \leq 2q} 
= \boldsymbol V^\top_{2q}(\boldsymbol \xi_1,\ldots,\boldsymbol \xi_m)
{\boldsymbol\Lambda},\quad\boldsymbol\Lambda
\isdef[\lambda_1,\ldots,\lambda_m]^\top.
\end{equation}

As mentioned in Remark \ref{rem:GaussQuad_1}, for $d=1$, Gaussian quadrature rules are 
efficient for the moment-matching problem since they use the optimal number of scenarios 
that match moments for the highest polynomial degree. In multiple dimensions, however, 
generalizing Gaussian quadrature is computationally challenging, and thus retrieving the 
optimal number of scenarios for the EMP is rather challenging. In particular, with regard 
to any representing measure $\mu$ for a truncated sequence $\bs y$, we have, in general, 
that $|\supp(\mu)| \geq \rank(\bs M_{\bs y})$, see \cite{curtofialkow96, fialkow1999}. 

To obtain scenarios that
accurately reflect the data while remaining computationally efficient, we consider a 
relaxed version of the EMP and determine the vector
\(\boldsymbol y^\star\) such that
\begin{equation}\label{eq:relaxedmm}
\qquad \big\|\boldsymbol y^\star - \widehat{\boldsymbol y}\big
\|\leq\varepsilon
\end{equation}
for a given tolerance \(\varepsilon>0\) and a given norm \(\|\cdot\|\) on \(\Rbb^{m_{2q}}\). 

Note that with the assumption of the scenarios as a subset of the samples, 
we can write $\bs y^\star = \bs V^\top_{2q}(\bs x_1,\ldots,\bs x_N) \bs w$  
where $\bs w\in\Rbb^N$ has at most $m$ non-zero entries. Thus, constructing $\bs y^\star$ 
in \eqref{eq:relaxedmm} is equivalent to a particular strategy of column subsampling 
of $\bs V^\top_{2q}(\bs x_1,\ldots, \bs x_N)$. This can be performed by ensuring only 
certain entries of $\bs w$ are non-zero. However, this is not enough to ensure the 
convexity of the weights. With this in mind, to simultaneously ensure that the weights are 
normalized and positive, we solve the problem in two stages. First, setting the norm in 
\eqref{eq:relaxedmm} to be the Euclidean norm, $\|\cdot\|_2$, we extract the scenarios following 
a greedy paradigm that constructs a \emph{sparse weight vector} $\bs w\in \R^N$ such that 
\begin{equation}\label{eq:unconstrained_optimization} 
    \begin{split}
       \big\| \boldsymbol V^\top_{2q}(\bs x_1,\ldots,\bs x_N)
       \boldsymbol w - \widehat{\boldsymbol y} 
         \big\|_2 \leq \varepsilon.
    \end{split}
\end{equation}
Note that the Euclidean norm corresponds 
to the mean squared error in the approximation. Alternatively, one can set the 
$\sup$-norm for the worst-case error. 
Second, we retrieve the corresponding probabilities by
enforcing the simplex constraints $\Delta 
\isdef \{\bs w : \bs w \geq \bs 0, \bs 1^\top \bs w = 1\}$ with 
a standard algorithm in Section \ref{sec:weightretrieval}.

For now, we focus on Problem \ref{eq:unconstrained_optimization} for the extraction of 
good representative scenarios. Although \eqref{eq:unconstrained_optimization} is trivially 
minimized at the constant vector with element $1/N$, nevertheless, it is underdetermined 
in the regime $m_{2q} \ll N$ when we have a large number of samples. Hence, we need 
a computational framework that helps us in choosing a good representative subspace 
of $\operatorname{Im}(\bs V_{2q}^\top\big(\bs x_1,\ldots,\bs x_N)\big)$. Towards that end, 
we propose to use an RKHS framework that allows us to do the above. 

\subsection{Reformulation in RKHS} 
Our main reason for the reformulation in an RKHS setting is as follows: with our 
assumption that the finite atomic target measure $\Pbb^\star \ll \widehat \Pbb$, 
we are essentially looking to extract the optimal support set for $\Pbb^\star$ when solving Problem~\ref{eq:unconstrained_optimization}. We want $\Pbb^\star$ and $\widehat \Pbb$ to be
pointwise close. One key property of an RKHS is that, if two functions are close in the RKHS norm,
then they are pointwise close as well. This works to our advantage since we want to express 
the empirical polynomial moments using fewer scenarios.

For completeness, we provide the definition of the RKHS below, 
see \cite{pau_rag_16}, \cite{RKHS_book}, for further details.

\begin{definition}\label{def:RKHS} 
Let \(\big(\mathcal{H},\langle{\cdot},{\cdot}\rangle_{\mathcal{H}}\big)\) 
be a Hilbert space of real-valued functions on a non-empty set $\Omega$. 
A function \(\kernel\colon\Omega \times \Omega \longrightarrow \Rbb\) is a 
\emph{reproducing kernel} of the Hilbert space $\mathcal{H}$ iff 
\begin{equation}\label{eq:rep_kernel}
\begin{split}
 &\kernel(\boldsymbol x, \cdot) \in \mathcal{H}\quad\text{for all }\boldsymbol x \in \Omega\\
&\big\langle{\kernel(\boldsymbol x, \cdot)},h\big\rangle_{\mathcal{H}} = h(\boldsymbol x)
\quad\text{for all } \boldsymbol x \in \Omega,\ h \in \mathcal{H}.
\end{split}
\end{equation}
The last condition of Equation \eqref{eq:rep_kernel} is called the 
\emph{reproducing property}. If these two properties hold, then 
\(\big(\mathcal{H},\langle{\cdot},{\cdot}\rangle_{\mathcal{H}}\big)\) 
is called a \emph{reproducing kernel Hilbert space}.
\end{definition}
In practice, given a finite set of data samples $X \isdef \{\bs x_1,\dots, \bs x_N\} 
\subset \Omega \subset \R^d$, we work in the subspace of kernel translates 
$\Hcal_{X} \isdef \spn\{\kernel(\bs x_i,\cdot): \bs x_i \in X\}$. 

We wish to embed the polynomial moments of the samples in the RKHS such that it 
is spanned by the columns of $\bs V_{2q}^\top\big(\bs x_1,\ldots,\bs x_N)$. Hence, 
we consider \(\Pscr_{2q}(X)\) endowed with the \(L^2_{\widehat{\Pbb}}\)-inner
product.
We can set the basis of $\Hcal_X$ to be the standard basis 
$\bs U(\bs x) \isdef \bs \tau_{2q}^\top(\bs x)$, where each column vector is now 
treated as a function in $\Hcal_X$. We also define 
$\bs U \isdef [\bs U(\bs x_1), \ldots, \bs U(\bs x_N)]$ to be the basis functions 
evaluated at the samples, which is just $\bs V^\top$. However, the standard bases 
of kernel translates are known to have poor conditioning, see \cite{DeMarchi2010}. 
Furthermore, in most practical applications involving large data samples, 
we typically have $m_{2q} \ll N$. Therefore, we seek a suitable \emph{data-dependent} 
basis for $\Hcal_X$ such that it is amenable to the task at hand. For this, we
introduce an appropriate representation of the reproducing kernel.
Toward that end, we first compute the Gram matrix as follows, 
\[
\langle \bs U, \bs U^\top \rangle_{\Hcal_X} = \langle \bs \tau_{2q}^\top, 
\bs \tau_{2q}\rangle_{L^2_{\widehat{\Pbb}}(X)} 
= \bigg[\int_{X}\bs x^{\bs \alpha + \bs \beta}
\d\widehat{\Pbb}\bigg]_{|\bs \alpha|, |\bs \beta| \leq 2q} 
= \frac{1}{N}\bs V^\top_{2q}(\bs x_1,\ldots,\bs x_N)\bs V_{2q}(\bs x_1,\ldots,\bs x_N),
\] 
which turns out to be the empirical moment matrix of order $2q$, which we will 
write henceforth as $\bs M$. 
There holds for the reproducing kernel restricted to $\Hcal_{X}$ that
\begin{equation}\label{eq:kernel}
 \kernel(\boldsymbol x, \boldsymbol x') = 
 \bs U(\boldsymbol x)^\top  \boldsymbol M^\dagger   \bs U(\boldsymbol x').
\end{equation}
Herein, ${\boldsymbol M}^{\dagger}$ denotes the Moore-Penrose inverse of $\bs M$. 
Associated with the kernel $\kernel$, we introduce the canonical feature map 
$\Phi: X \to \Hcal_X, \Phi(\bs x) \mapsto \kernel(\bs x, \cdot)$ that embeds the 
data from $X$ into the Hilbert space of functions. We denote the 
basis of kernel translates by
\begin{equation}\label{eq:canonical_basis}
    \bs \Phi(\bs x) \isdef \big[\kernel(\bs x,\bs x_1),
    \ldots,\kernel(\bs x,\bs x_N)\big].
\end{equation}

For the sake of completeness, we show that the kernel defined in \eqref{eq:kernel} 
is indeed the \emph{reproducing kernel} on $\Hcal_{X}$.
\begin{theorem}\label{thm:rep_kernel}
    The function $\kernel\colon X \times X \to \Rbb$  as defined in \eqref{eq:kernel}
    is a symmetric and positive type function. Moreover, $\kernel$ has the reproducing
property on $X$,
\begin{equation}\label{eq:rep_property}
\big\langle{\kernel(\boldsymbol x_i, \cdot)},{p}\big\rangle_{\Hcal_X} 
= p(\boldsymbol x_i) \quad\text{for all }
\boldsymbol x_i \in X, p \in \Hcal_X.
\end{equation}
\end{theorem}
\begin{proof}
For the sake of a lighter notation, we will henceforth write \({\boldsymbol V}\isdef{\boldsymbol V}_{2q}({\boldsymbol x}_1,\ldots,
{\boldsymbol x}_N) \in \Rbb^{N \times m_{2q}}\). 
   We introduce the kernel matrix
\begin{equation}\label{eq:kernelmatrix}
    \boldsymbol K\isdef \big[\kernel(\boldsymbol x_i, \boldsymbol x_j)\big]_{i,j=1}^{N} = 
    \bs U^\top \boldsymbol M^\dagger  \bs U = 
    \boldsymbol V \boldsymbol M^\dagger
\boldsymbol V^\top \in \Rbb^{N \times N}.
\end{equation}
The matrix $\boldsymbol K$ is symmetric, which follows from the fact that
 \({\boldsymbol M}^\dag\) is symmetric.
There holds
\(
{\boldsymbol M}^\dag=N{\boldsymbol V}^\dag\big({\boldsymbol V}^\top\big)^\dag.
\) Note that $\rank \boldsymbol K=\rank \boldsymbol M$. Furthermore, we have
\[
{\boldsymbol a}^\top{\boldsymbol K}{\boldsymbol a}
=N{\boldsymbol a}^\top{\boldsymbol V}
{\boldsymbol V}^\dag\big({\boldsymbol V}^\top\big)^\dag
{\boldsymbol V}^\top{\boldsymbol a}=
N\big\|\big({\boldsymbol V}^\top\big)^\dag{\boldsymbol V}^\top
{\boldsymbol a}\big\|_2^2\geq 0\quad\text{for all }
{\boldsymbol a}\in\Rbb^{N}
\] Hence, the matrix ${\boldsymbol K}$
is symmetric and positive semi-definite, i.e., the kernel
$\kernel$ is a symmetric and positive type function. In order to prove the reproducing property, we first recall that the Moore-Penrose inverse 
satisfies \({\boldsymbol A}{\boldsymbol A}^\dag{\boldsymbol A}={\boldsymbol A}\),
\(({\boldsymbol A}{\boldsymbol A}^\top)^\dag
=({\boldsymbol A}^\top)^\dag{\boldsymbol A}^\dag\) as well as
\(({\boldsymbol A}{\boldsymbol A}^\dag)^\top
={\boldsymbol A}{\boldsymbol A}^\dag\) for any matrix \({\boldsymbol A}\).
From this, we directly infer
\begin{equation}\label{eq:V_rep_prop}
\begin{aligned}
{\boldsymbol V}\big({\boldsymbol M}^\dagger{\boldsymbol M}\big)
={\boldsymbol V}\Big(\big({\boldsymbol V}^\top{\boldsymbol V}\big)^\dag
{\boldsymbol V}^\top{\boldsymbol V}\Big)
=\Big(\big({\boldsymbol V}{\boldsymbol V}^\dag\big)
\big({\boldsymbol V}{\boldsymbol V}^\dag\big)^\top\Big){\boldsymbol V}
=\big({\boldsymbol V}{\boldsymbol V}^\dag\big)\big({\boldsymbol V}
{\boldsymbol V}^\dag\big){\boldsymbol V}={\boldsymbol V}.
\end{aligned}
\end{equation}
Let \(p \in \Hcal_X \equiv \Pscr_{2q}(X)\), i.e., \(p(\boldsymbol x_i) 
= \boldsymbol \tau_{2q}(\boldsymbol x_i) \boldsymbol p\) for any 
coefficient vector \(\boldsymbol p \in \Rbb^{m_{2q}}\) and any $\bs x_i \in X.$ 
Hence,
\[
 \big\langle{\kernel(\boldsymbol x_i, \cdot)}, {p}\big\rangle_{\Hcal_X} 
 =  \boldsymbol \tau_{2q}(\boldsymbol x_i) \boldsymbol M^\dag 
 \big\langle{\boldsymbol \tau_{2q}^\top},{\boldsymbol \tau_{2q}}
 \big\rangle_{L^2_{\widehat \Pbb}(X)} \boldsymbol p 
 = \boldsymbol \tau_{2q}(\boldsymbol x_i) \boldsymbol M^\dag \boldsymbol M \boldsymbol p 
 = \bs \tau_{2q}(\boldsymbol x_i) \boldsymbol p = p(\boldsymbol x_i),
\]
where we use \eqref{eq:V_rep_prop} and that \(\big[\boldsymbol V_{i,j}\big]_{j=1}^{m_{2q}} 
= \boldsymbol \tau_{2q}(\boldsymbol x_i).\) 
\end{proof}
\begin{corollary}\label{res:KV}
    A consequence of the reproducing property is $\bs K\bs V = N \bs V$.
\end{corollary}

We close this paragraph by providing an algorithm for the computation of 
an orthonormal basis for $\Hcal_X$. Any general data-dependent basis turns out 
to be defined via a factorization of the Gram matrix defined by these data, 
see \cite{PS2011}. Various matrix factorizations give rise to different bases 
with different properties. For our purpose, we apply the pivoted Cholesky 
factorization, see \cite{HPS12}, as in Algorithm~\ref{algo:pivChol}
on the positive semi-definite moment matrix $\boldsymbol M$. With a sufficiently 
low error tolerance $\varepsilon$, we obtain $m_{2q} \times r$ 
matrices $\bs B_{\bs M}$ and $\bs L_{\bs M}$ with $r = \rank(\bs M) \leq m_{2q}$ 
such that
\[\trace({\boldsymbol M}-{\boldsymbol L_{\bs M}}{\boldsymbol L^\top_{\bs M}}) 
\leq \varepsilon, \quad \bs B^\top_{\bs M}\bs L_{\bs M} = \bs I_{r}, 
\quad \bs M\bs B_{\bs M} = \bs L_{\bs M}.
\]
We particularly
have \({\boldsymbol M}^\dag ={\boldsymbol B_{\bs M}}{\boldsymbol B^\top_{\bs M}}\). 
The basis transformation is an essential byproduct of efficiently solving the 
low-rank approximation of the unconstrained optimization problem at a large scale.
\begin{algorithm}\caption{Pivoted Cholesky Decomposition (\cite{HPS12})}
\label{algo:pivChol}
\begin{flushleft}
\begin{tabular}{ll}
\textbf{input:}  & symmetric and positive semi-definite
matrix ${\boldsymbol M}\in\mathbb{R}^{m_{2q} \times m_{2q}}$, \\
& tolerance \(\varepsilon\geq0\)\\
\textbf{output:} & low-rank approximation
\({\boldsymbol M}\approx{\boldsymbol L_{\bs M}}{\boldsymbol L^\top_{\bs M}}\)\\
&and
biorthogonal basis \({\boldsymbol B_{\bs M}}\)
such that \({\boldsymbol B^\top_{\bs M}}{\boldsymbol L_{\bs M}}={\boldsymbol I}_r\)
\end{tabular}
\end{flushleft}
\begin{algorithmic}[1]
\smallskip
\State Initialization: set $r\isdef 1$,
 ${\boldsymbol d}\isdef \operatorname{diag}({\boldsymbol M
})$, \({\boldsymbol L_{\bs M}}\isdef[\,]\),
 \({\boldsymbol B_{\bs M}}\isdef[\,],\)
$\operatorname{err}\isdef \|{\boldsymbol d}\|_{1}$
\State\textbf{while} \(\operatorname{err}>\varepsilon\)
\State\StateIndent determine
$\pi(r)\isdef \operatorname{arg}\max_{1\le i\le m_{2q}} d_i$
\State\StateIndent compute \[
\boldsymbol\ell_r
\isdef \frac{1}{\sqrt{d_{\pi(r)}}}\Big({\boldsymbol M}-{\boldsymbol L_{\bs M}}
{\boldsymbol L^\top_{\bs M}}\Big)\boldsymbol e_{\pi(r)}\quad\
\text{and}\quad\boldsymbol b_r
\isdef \frac{1}{\sqrt{d_{\pi(r)}}}\Big({\boldsymbol I}-{\boldsymbol B_{\bs M}}
{\boldsymbol L^\top_{\bs M}}\Big)\boldsymbol e_{\pi(r)}
\]
\State\StateIndent set \({\boldsymbol L_{\bs M}}
\isdef [{\boldsymbol L_{\bs M}},{\boldsymbol\ell}_r], \quad {\boldsymbol B_{\bs M}}
\isdef [{\boldsymbol B_{\bs M}},{\boldsymbol b}_r]\)
\State\StateIndent set \({\boldsymbol d}
\isdef {\boldsymbol d}-{\boldsymbol\ell}_r\odot
{\boldsymbol\ell}_r\), where $\odot$ denotes the Hadamard product
\State\StateIndent set \(\operatorname{err}
\isdef\|{\boldsymbol d}\|_1\), \(r\isdef r+1\)
\end{algorithmic}
\end{algorithm}

Using the bi-orthogonal basis $\bs B_{\bs M}$ arising out of the pivoted 
Cholesky algorithm, we define the matrix
\begin{equation}\label{eq:isometry}
{\boldsymbol Q}\isdef
\frac{1}{\sqrt{N}}{\boldsymbol V}{\boldsymbol B_{\bs M}}\in\Rbb^{N\times r}    
\end{equation}
that satisfies \({\boldsymbol Q}^\top{\boldsymbol Q}={\boldsymbol I}_r\).
The matrix ${\boldsymbol K}$ 
can be decomposed as 
\begin{equation}
    {\boldsymbol K} = N{\boldsymbol Q}{\boldsymbol Q}^\top.
\end{equation}
The above decomposition shows that the kernel matrix is simply the projection 
onto the columns spanned by the isometry and in particular, 
we have $\operatorname{Im}(\bs V) = \operatorname{Im} (\bs Q) 
=  \operatorname{Im}(\bs K)$.
With the above insights, we prove the following result, which will help 
in the numerical solution of problem \eqref{eq:unconstrained_optimization}.
\begin{theorem} Defining $\widetilde{\bs y} 
\isdef \frac{1}{\sqrt{N}}\bs B^\top_{\bs M}\widehat{\bs y}$, 
we have the following conditioning relation,
\begin{equation*}
   \frac{\sigma_{\text{min}}(\boldsymbol M^\dagger)}{N} 
   \cdot \big\|{\boldsymbol V}^\top{\boldsymbol w}
-\widehat{\boldsymbol y}\big\|_2^2 \hspace{3pt} \leq \hspace{3pt} 
\big\|\boldsymbol Q^\top \boldsymbol w - \widetilde{\boldsymbol y}
\big\|_2^2 \leq \frac{\sigma_{\text{max}}
(\boldsymbol M^\dagger)}{N} \cdot\big\|{\boldsymbol V}^\top{\boldsymbol w}
-\widehat{\boldsymbol y}\big\|_2^2.
\end{equation*}
\end{theorem}
\begin{proof}
    For the right-side inequality, we have 
    \begin{align*}
        \big\|\boldsymbol Q^\top \boldsymbol w - \widetilde{\boldsymbol y}\big\|_2^2 = 
        \bigg\|\frac{1}{\sqrt{N}}\bs B^\top_{\bs M}\big(\bs V^\top  
        \boldsymbol w - \widehat{\boldsymbol y}\big)
        \bigg\|_2^2 \leq \frac{1}{N}\big\|\bs B_{\bs M}^\top\big\|_2^2 
        \cdot \big\|\bs V^\top  \boldsymbol w - \widehat{\boldsymbol y}
        \big\|_2^2 
    \end{align*} 
    and noticing that $\|\bs B_{\bs M}^\top\|_2^2 
    = \sigma_{\text{max}}(\bs B_{\bs M}\bs B^\top_{\bs M}) 
    = \sigma_{\text{max}}(\bs M^\dagger)$. For the left-side inequality, we have
    \begin{align*}
    &\big\|{\boldsymbol V}^\top{\boldsymbol w} - \widehat{\boldsymbol y}\big\|_2^2 = 
    \bigg\|{\boldsymbol V}^\top\bigg({\boldsymbol w} - \frac{1}{N}\bs 1\bigg)\bigg\|_2^2 
    = \bigg\|\frac{1}{N}{\boldsymbol V}^\top\bs K\bigg({\boldsymbol w} 
    - \frac{1}{N}\bs 1\bigg)\bigg\|_2^2 = \bigg\|{\boldsymbol V}^\top\bs Q\bs Q^\top\bigg({\boldsymbol w} 
    - \frac{1}{N}\bs   1\bigg)\bigg\|_2^2\\
    &= \big\|{\boldsymbol V}^\top\bs Q\bs Q^\top{\boldsymbol w} 
    - \bs V^\top\bs Q\widetilde{\bs y}\Big\|_2^2 \leq \big\|\bs V^\top\big\|_2^2 \cdot 
    \big\|\bs Q \big( \bs Q^\top\boldsymbol w - \widetilde{\boldsymbol y}\big)\big\|_{2}^2 
    = \big\|\bs V^\top\big\|_2^2 \cdot \big\|\bs Q^\top \boldsymbol w 
    - \widetilde{\boldsymbol y}\big\|_{2}^2
    \end{align*}
where the second equality on the first line follows from \eqref{res:KV}, 
and the last equality is due to $\bs Q$ being an isometry. Thus, 
using the decomposition of the moment matrix 
$\boldsymbol M = (\boldsymbol V^\top\boldsymbol V)/N$, 
we have that $\|\boldsymbol V^\top\|_2^2 = N\sigma_{\text{max}}
(\boldsymbol M) = N/\sigma_{\text{min}}(\boldsymbol M^\dagger)$. 
\end{proof}
\begin{remark} The normal equations of Problem ~\ref{eq:unconstrained_optimization} 
reads $\bs V\bs V^\top \bs w = \bs V\widehat{\bs y}$, so that the condition number of 
the system is $\kappa(\bs V\bs V^\top) = \kappa(\bs V^\top\bs V) = \kappa(\bs M)$. 
Now, the condition number of Problem ~\ref{eq:modified_unconstrained_optimization} 
is $\kappa(\bs Q\bs Q^\top) = \kappa(\bs Q^\top\bs Q) = 1$. Hence, the above proposition 
shows that, up to $\kappa(\bs M)$, 
\eqref{eq:unconstrained_optimization} is equivalent to 
\begin{equation}\label{eq:modified_unconstrained_optimization}
         \quad \big\| \boldsymbol Q^\top \boldsymbol w
         - \widetilde{\boldsymbol y}\big\|_2,
\end{equation}
which we will henceforth consider for the greedy extraction of scenarios, 
instead of \eqref{eq:unconstrained_optimization}.
\end{remark}
In the next result, we collect additional properties of
$\boldsymbol Q$, in particular in connection with the above
optimization problem.

\begin{lemma}\label{lem:qprops}
 Matrix $\boldsymbol Q$ satisfies the relation
 \begin{equation}\label{eq:Qy}
  \boldsymbol Q \widetilde{\boldsymbol y}=\frac{1}{N}\boldsymbol 1.
 \end{equation} 
From the normal equations \({\boldsymbol Q}
{\boldsymbol Q}^\top{\boldsymbol w}={\boldsymbol Q}
\widetilde{\boldsymbol y}\),  vector
\(
  \boldsymbol w=\frac{1}{N}\boldsymbol 1
\)
is a solution of \eqref{eq:modified_unconstrained_optimization}.
\end{lemma}
\begin{proof}
To show
\eqref{eq:Qy}, we first note that \(\widehat{\boldsymbol y}\) can be
identified with the first column of the moment matrix, i.e.,
\(\widehat{\boldsymbol y} = \boldsymbol M \boldsymbol e_1 
= \frac{1}{N}\boldsymbol V^\top \boldsymbol V \boldsymbol e_1\).
Hence, we can write
\begin{equation}
 \boldsymbol Q \widetilde{\boldsymbol y}=\frac{1}{N}
 \boldsymbol V \boldsymbol B_{\bs M} \boldsymbol B_{\bs M} ^{\top}\boldsymbol M
 \boldsymbol e_1 = \frac{1}{N}
 \boldsymbol V \boldsymbol M^\dag \boldsymbol M \boldsymbol e_1 
 = \frac{1}{N} \boldsymbol V \boldsymbol e_1 = \frac{1}{N} 
 \boldsymbol 1,
\end{equation} 
where we exploit the reproducing property of the kernel on the sample set, cp.\ \eqref{eq:V_rep_prop}, and the
identity
\(\boldsymbol M^\dagger =\boldsymbol B_{\bs M} \boldsymbol B_{\bs M }^\top \).
We have as a corollary 
\[
\boldsymbol Q \widetilde{\boldsymbol y} 
= \frac{1}{N} \boldsymbol 1
=
\boldsymbol Q \boldsymbol Q^\top \frac{1}{N} \boldsymbol 1, 
\]
again by \eqref{eq:V_rep_prop} and the fact that 
\(\boldsymbol V \boldsymbol e_1 =
 \boldsymbol 1\). Hence, the vector \(
  \boldsymbol w=\frac{1}{N}\boldsymbol 1
\) is a solution to the normal equations.
\end{proof}

\section{Algorithms for scenario representation}\label{sec:scenarioextraction}
In this section, we present the algorithms for the two stages as mentioned 
in Section \ref{sec:EMP}. The first algorithm, 
\emph{orthogonal matching pursuit} (OMP), see \cite{PS2011},
is devised to extract the scenarios 
in a greedy manner. The second is 
the \emph{alternating direction method of multipliers} (ADMM), see \cite{ADMM},
 and is used for retrieving the corresponding probabilities.
\subsection{Step 1: Scenario extraction}
For the scenario extraction problem, we first note that a minimum-norm solution of Problem ~\ref{eq:modified_unconstrained_optimization} satisfies the normal equations
\begin{equation}\label{eq:normal_equations}
    \bs Q \bs Q^\top\bs w = \bs Q \widetilde{\bs y} = \frac{1}{N} \bs 1, \quad \text{which implies} \quad \bs K \bs w = \bs 1,
\end{equation}
where the second equality is due to Lemma \ref{lem:qprops}.
Suppose that $h\in\Hcal_X$ is a function such that its evaluation on the samples gives the vector of ones i.e.~$[h(\bs x_i)]_{i=1}^N = \bs 1$. Any $\bs w\in \R^N$ satisfying the equation $ \bs K \bs w = \bs 1$ defines an interpolant $s_{h} \in \Hcal_X$ which can be written as 
\[
s_{h}(\bs x) = \sum_{i=1}^N w_i\,\kernel(\bs x, \bs x_i) = \bs \Phi(\bs x)\bs w \quad \text{for all } \bs x \in X, 
\]
where $\bs \Phi(\bs x)$ is from \eqref{eq:canonical_basis}.
More generally, any interpolant that exactly matches $m$ entries of $h$ has the form 
\begin{equation}\label{eq:interpolant}
    s_{h,m}(\bs x) \isdef \sum_{j=1}^m w_j\,\kernel(\bs x, \bs \xi_j) \quad \text{for all } \bs x \in X,
\end{equation}
where $\{\bs \xi_1,\ldots,\bs \xi_m\} \subset X$. The interpolant $s_{h,m} \in \spn\{\kernel(\cdot, \bs \xi_1),\ldots,\kernel(\cdot,\bs \xi_m)\}$ and the problem of choosing a few number of scenarios from a large set of samples is therefore equivalent to choosing a subset from the dictionary $\{\kernel(\cdot,\bs x_1),\ldots,\kernel(\cdot,\bs x_N)\}$. Hereby, we resort to a greedy subsampling of the columns of ${\boldsymbol K}$ that can
approximate its column space within a certain error tolerance. With this
objective in mind, we adapt the notion of the pivoted Cholesky decomposition of ${\boldsymbol K}$
into the setting of Algorithm~\ref{algo:OMP} as below.  

\begin{algorithm}\caption{Orthogonal matching pursuit (OMP)}
\label{algo:OMP}
\begin{flushleft}
\begin{tabular}{ll}
\textbf{input:}  & symmetric and positive semi-definite
matrix ${\boldsymbol K}\in\mathbb{R}^{N \times N}$,\\
& tolerance \(\varepsilon > 0\)\\
\textbf{output:} &index set $\pi$ corresponding to the sparse scenarios $\Xi\isdef \{\bs \xi_1,\ldots, \bs \xi_m\} \subset \Rbb^d $, \\
&low-rank approximation \({\boldsymbol K}\approx{\boldsymbol L_m}
{\boldsymbol L}^\top_m\)\\
&and bi-orthogonal basis \({\boldsymbol B_m}\)
such that \({\boldsymbol B}^\top_m{\boldsymbol L_m}={\boldsymbol I}_m\) 
\end{tabular}
\end{flushleft}
\begin{algorithmic}[1]
\smallskip
\State Initialization: set $m\isdef 1$, \({\boldsymbol L_0}\isdef[\,]\),
 \({\boldsymbol B_0}\isdef[\,], \pi\isdef [\,],
 \boldsymbol d_0 = \diag(\boldsymbol K), \boldsymbol h_0 = \boldsymbol 1\),
$\operatorname{err}= 1$
\State\textbf{while} \(\operatorname{err}>\varepsilon\)
\State\StateIndent determine
\(\pi(m) \isdef \operatorname{arg}\max_{1\le i\le N} |h_{m-1,i}|\)
\State\StateIndent \(\operatorname{\pi} \isdef [\pi,  \pi(m)]\)
\State\StateIndent compute \[
\boldsymbol\ell_m
\isdef \frac{1}{\sqrt{d_{m-1,\pi(m)}}}\Big({\boldsymbol K}-{\boldsymbol L_{m-1}}
{\boldsymbol L}^\top_{m-1}\Big)\boldsymbol e_{\pi(m)}\quad\
\text{and}\quad\boldsymbol b_m
\isdef \frac{1}{\sqrt{d_{m-1,\pi(m)}}}\Big({\boldsymbol I}-{\boldsymbol B_{m-1}}
{\boldsymbol L}^\top_{m-1}\Big)\boldsymbol e_{\pi(m)}\]
\State\StateIndent set \({\boldsymbol L_m}
\isdef [{\boldsymbol L_{m-1}}, {\boldsymbol\ell}_m], \quad {\boldsymbol B_m}
\isdef [{\boldsymbol B_{m-1}}, {\boldsymbol b}_m]\)
\smallskip
\State\StateIndent set \({\boldsymbol h_{m}}\isdef {\boldsymbol h_{m-1}} 
- ({\boldsymbol b^\top_m}{\boldsymbol h_0}){\boldsymbol\ell}_m\)
\smallskip
\State\StateIndent set \(\boldsymbol d_m \isdef \boldsymbol d_{m-1} 
- \boldsymbol\ell_m \odot \boldsymbol\ell_m\), where $\odot$ denotes the
Hadamard product
\smallskip
\State\StateIndent set \(\operatorname{err}
\isdef\left \|{\boldsymbol h_{m}}\right\|_2/\|\bs h_0\|_2\),
\(m\isdef m+1\)
\end{algorithmic}
\end{algorithm}

The bi-orthogonal basis $\bs B_m$ issuing from the Algorithm \ref{algo:OMP} can 
be used to define an orthonormal set of basis functions. We define the 
\emph{Newton basis} as $\bs N(\bs x) \isdef \bs \Phi(\bs x)\bs B_m$, 
see \cite{PS2011}. The Newton basis evaluated on the set of samples $X$ gives 
\begin{equation}\label{eq:newton_basis}
    \bs N \isdef \big[N(\bs x_i)\big]_{i=1}^N 
    = \big[\bs \Phi(\bs x_i)\bs B_m\big]_{i=1}^N = \bs K \bs B_m = \bs L_m. 
\end{equation}
Hence, we have that the column vectors of $\bs L_m$ of Algorithm~\ref{algo:OMP}, 
i.e., $\bs \ell_1,\ldots,\bs \ell_m$ is just the Newton basis evaluated at 
the sample points. Denoting the $m$ basis functions as $N_1,\ldots,N_m$, 
we have that
\[
\operatorname{Im}(\bs L_m) = \spn\{N_1,\ldots,N_m\} 
= \spn\{\kernel(\cdot,\bs x_{\pi(1)}),\ldots,\kernel(\cdot, \bs x_{\pi(m)})\} 
\subset \spn\{\kernel(\cdot,\bs x_1),\ldots,\kernel(\cdot,\bs x_N)\} = \Hcal_X.
\]
Furthermore, they form an orthonormal system in $\Hcal_X$ i.e.$\langle 
N_i,N_j\rangle_{\Hcal_X} = \delta_{i,j}$ for $1\leq i,j\leq m$, since 
\[
\langle \bs N^\top, \bs N \rangle_{\Hcal_X} = \bs B_m^\top\langle \bs \Phi^\top, 
\bs \Phi\rangle_{\Hcal_X} \bs B_m = \bs B_m^\top \bs K \bs B_m 
= \bs B_m^\top \bs L_m = \bs I_m.
\]
Therefore, the Newton basis functions $N_1,\ldots,N_m$ are an orthonormal basis 
for the space of the kernel translates
$\kernel(\cdot,\bs x_{\pi(1)}),\ldots,\kernel(\cdot, \bs x_{\pi(m)})$. 
This basis is adaptively constructed by means of a Gram-Schmidt procedure
to represent the function \(h\in\Hcal\), i.e.,
\[
h_m = h_{m-1} - N_m \langle N_m, h\rangle_{\Hcal_X} 
= h-\sum_{j=1}^m N_j \langle N_j, h\rangle_{\Hcal_X}\defis h-P_\Fcal h.
\]

From Algorithm \ref{algo:OMP} line 6, the vector of evaluations of $h_m$ has 
the form $\bs h_m = (\bs I - \bs L_{m}\bs B_m^\top)\bs h_0$. Therefore, the component 
of $h_0$ in the direction of $N_{m}$ i.e., $N_m \langle N_m, h_0\rangle_{\Hcal_X}$ 
has the vector form $\bs \ell_m (\bs b_m^\top \bs h_0)$. Line $7$ of 
Algorithm~\ref{algo:OMP} computes exactly this error between $h_0$ and its 
orthogonal projection onto the subspace spanned by $N_1,\ldots,N_m$.
By standard arguments, we have that the mean-squared error satisfies
\begin{equation}\label{eq:error}
     \frac{1}{N}\sum_{i=1}^N \big(h(\bs x_i) - h_{\Fcal}(\bs x_i)\big)^2 
     \leq \frac{1}{N}\trace(\bs K - \bs L_m\bs L_m^\top) \|h\|_{\Hcal_X}^2.
\end{equation}
From the right-hand side, we infer that the mean-squared error of approximation 
is controlled by the reduction in the trace error by the low-rank approximation 
of the kernel matrix $\bs K$. Hence, for any general $h\in \Hcal_X$, the pivoting 
strategy of the pivoted Cholesky of \cite{HPS12} is the optimal. However,
the strategy in Algorithm \ref{algo:OMP} might result in a lower rank for
the task at hand. The total cost of performing the $m$ steps of Algorithm \ref{algo:OMP} is $\Ocal(m^2 N)$, see \cite{filipovic2023adaptive}.

\subsection{Step 2: Retrieval of probability weights}\label{sec:weightretrieval} 
Having extracted the set of $m$ scenarios $\Xi \isdef \{\bs \xi_1,\ldots,\bs \xi_m\} 
= \{\bs x_{\pi(1)},\ldots,\bs x_{\pi(m)}\}$, we proceed to enforce the 
probability simplex constraints. We deploy
the ADMM algorithm, see \cite{ADMM, Proximal_algorithms}. This approach can be 
used for projecting least-squares solutions of linear systems onto convex sets. 
Denote the set of the two simplex constraints as
\begin{equation}
    \Delta_{+} \isdef \big\{\boldsymbol w \in \Rbb^m : 
    \boldsymbol w \geq \boldsymbol 0 \big\}, \quad  \Delta_{1} 
    \isdef \big\{\boldsymbol w\in \Rbb^m: \boldsymbol 1^\top \boldsymbol w = 1\big\}.
\end{equation}
We now use the ADMM algorithm to solve the following problem
\begin{equation}\label{eq:constrained_optimization}
\begin{split}
    \bs \Lambda &= \underset{\bs w \in \R^m}{\argmin}
    \big\|\bs V^\top_{2q}(\bs \xi_1,\ldots,\bs \xi_m)\bs w - \widehat{\bs y}\big\|_2 \\
    &\text{subject to: } \Delta \isdef \Delta_{+} \cap \Delta_{1}.
\end{split}
\end{equation}
The solution to the above problem gives us the corresponding probabilities $\lambda_j$ for the scenarios $\bs \xi_j$ for $1\leq j\leq m$. 

The steps for solving the EMP are summarized in Algorithm \ref{algo:scenario_rep}
\begin{algorithm}
\caption{Final algorithm for scenario representation}
\label{algo:scenario_rep}
\begin{flushleft}
\begin{tabular}{ll}
\textbf{input:}  & set of samples $X \isdef \{\bs x_1,\ldots,\bs x_N\}$ \\
\textbf{output:} &  \(m\) scenarios \(\Xi \isdef \big\{\boldsymbol \xi_1,\cdots,\boldsymbol\xi_m\big\} \subset \Rbb^d\) and their probabilities $\{\lambda_1,\ldots,\lambda_m\} \subset \Delta$
\end{tabular}
\end{flushleft}
\begin{algorithmic}[1]
\State Compute the empirical moment matrix $\bs M$.
\State Perform pivoted Cholesky of $\bs M$, cp.\eqref{algo:pivChol} to obtain $\bs B_{\bs M}$.
\State Set $\bs Q = \bs V\bs B_{\bs M}/\sqrt{N}$ and perform OMP on $N\bs Q\bs Q^\top$, cp.\eqref{algo:OMP}. 
\State Extract the scenarios as $\bs \xi_{j} = \bs x_{\pi(j)}, 1\leq j \leq m$ ($\pi$ is the set of pivots obtained from the OMP).
\State Retrieve $\lambda_j$ by solving \eqref{eq:constrained_optimization} using ADMM.
\end{algorithmic}
\end{algorithm}

\begin{remark} An extremely popular method to solve least squares problems, that is 
employed vastly in statistics and compressive sensing is the basis pursuit, i.e., 
\(\ell_1\)-regularized least squares, or LASSO. A well-known property of LASSO is that 
of subset selection i.e.,~it leads to sparse solutions of least squares problems, 
see \cite{StatisticalLearning}. It thus suggests itself as a  
viable alternative to solving the EMP. However, in our particular setting of 
finding optimal representation using scenarios, it does not lead to sparsity. 
Concretely, the minimizer of the basis pursuit problem  
\begin{equation*}
  \boldsymbol w ^{\star}=\underset{\boldsymbol w \in \Rbb^N}{\argmin }
  \frac{1}{2} \big\| \boldsymbol Q ^{\top}\boldsymbol w
  -\widetilde{\boldsymbol y}\big\|_2^2+\lambda \|\boldsymbol w\|_1
  =\underset{\boldsymbol w \in \R^N}{\argmin }
  \frac{1}{2} \big\| \boldsymbol V ^{\top}\boldsymbol w
  -\widehat{\boldsymbol y}\big\|_2^2+\lambda \|\boldsymbol w\|_1, \lambda > 0,
 \end{equation*}
has the constant form $\boldsymbol w^{\star}=c\boldsymbol 1$ with
\begin{equation*}
 c=\begin{cases}
    \frac{1-N \lambda}{N},  &\text{ if } \frac{1}{N}  > \lambda, \\
    0,&\text{ otherwise}.  
   \end{cases}
\end{equation*}
\begin{proof}
 We define the two respective cost functions \(g_{\lambda}, h_{\lambda}\colon\Rbb^N \longrightarrow \Rbb\) as
 \begin{equation*}
  g_{\lambda}(\boldsymbol w)\isdef \frac{1}{2}\big\|
  \boldsymbol Q ^{\top}\boldsymbol w- \widetilde{\boldsymbol y}\big\|_2^2
  +\lambda\|\boldsymbol w\|_1, \text{ and }h_{\lambda}(\boldsymbol w)
  \isdef \frac{1}{2}\big\|  \boldsymbol V ^{\top}\boldsymbol w
  - \widehat{\boldsymbol y}\big\|_2^2+\lambda\|\boldsymbol w\|_1.
 \end{equation*}
 Clearly, for given $\lambda>0$, both $g_{\lambda}$ and $h_{\lambda}$ 
 are convex functions. Therefore, $g_{\lambda}$ and $h_{\lambda}$ are 
 subdifferentiable over $\Rbb^N$ cp.\ \cite[Corollary 3.15]{Amir_Beck}. 
 We can compute them as, cf.\ \cite{Amir_Beck},
\[
 \partial g_{\lambda}(\boldsymbol w) = \boldsymbol Q  \boldsymbol 
 Q ^{\top}\boldsymbol w-\boldsymbol Q \widetilde{\boldsymbol y}
 +\lambda \partial  \|\boldsymbol w\|_1, \text{ and }
 \partial h_{\lambda}(\boldsymbol w) = \boldsymbol V  
 \boldsymbol V ^{\top}\left (\boldsymbol w-\frac{1}{N}\boldsymbol 1 \right )
 +\lambda \partial  \|\boldsymbol w\|_1, \quad \text{where}
 \]
\[ (\partial  \|\boldsymbol w\|_1)_i = \begin{cases}\sig w_i, & w_i\neq 0, \\
x \in [-1,1], & w_i=0  .  
\end{cases}
\]
From Lemma \ref{lem:qprops}, there holds $\boldsymbol Q \widetilde{\boldsymbol y}
=\boldsymbol Q \boldsymbol Q ^{\top}\boldsymbol 1 \frac{1}{N} =\frac{1}{N}\boldsymbol 1$. 
Hence, the subgradient $\partial g_{\lambda}$ reads
\begin{equation*}
 \partial g_{\lambda} \left(\frac{c}{N}\boldsymbol 1\right )
 =\bigg (\frac{c-1}{N}+\lambda \sig (c) \bigg)  \boldsymbol 1. 
\end{equation*}
Since $\boldsymbol 0\in \partial g_{\lambda}$ describes the optimality for this strictly convex problem, we can exclude $c<0$,  since then the subgradient will be strictly element-wise negative. However, suppose that $1/N>\lambda $ and $c>0$, then,  from the equation $\frac{c-1}{N}+\lambda \sig (c)=0$,   
$c=1-N\lambda$ proves to be optimal. For  $ 1/N\leq \lambda$, the subgradient evaluated at $\boldsymbol 0$ becomes
\[
 \partial g_{\lambda} \left(\boldsymbol 0\right )=\left (-\frac{1}{N}+\lambda | x|\right )\boldsymbol 1,
\]
where we can again deduce $\boldsymbol 0 \in \partial g_{\lambda}(\boldsymbol 0)$ with $x=1/(N\lambda)$. 
For $h_\lambda$, 
\[
 \partial h_{\lambda} \bigg(\frac{c}{N}\boldsymbol 1\bigg)=\boldsymbol V\boldsymbol V^\top 
 \bigg (\frac{c-1}{N}+\lambda \sig (c) \bigg)  \boldsymbol 1, 
 \text{ and } \partial h_{\lambda} \bigg(\boldsymbol 0\bigg)
 =\boldsymbol V\boldsymbol V^\top\bigg (-\frac{1}{N}+\lambda | x|\bigg)\boldsymbol 1,
\]
the cases are analogous. 
\end{proof}
Note that in the above proof, we do not assume any positivity of the weight vector \(\boldsymbol w\). 
The result therefore has far-reaching consequences for practical purposes. Foremostly, it shows that 
it is not possible to obtain optimal quadrature rules (with few atoms) in general that match the 
empirical moments of polynomials. This fact severely restricts the pool of candidate algorithms 
that can be deployed for efficient recovery of a few scenarios from the empirical moments. 
In particular, many popular first-order proximal gradient methods, like the FISTA, 
see \cite{FISTA, Nesterov}, 
POGM, see \cite{POGM, Adaptive_POGM}, etc.~to name a few cannot be used in this context.
\end{remark}


\section{Numerical experiments}\label{sec:numExp}
In this section, we perform numerical experiments to investigate the behavior of the 
proposed algorithms concerning the dimension, the number of data points, and the order 
of the maximum degree of the polynomial basis generating the moment matrix in question. 
We test the algorithm by \cite{Lasserre}, the maximum volume algorithm, see 
\cite{bittante,BSV10,BosSDeMarchiSommarivaVianello,GVL13,SOMMARIVA20091324}, the graded 
hard thresholding pursuit (GHTP), see \cite{GHTP}, the OMP Algorithm \ref{algo:OMP}, 
and finally the covariance scenarios described in Algorithm \ref{algo:cov_scenarios}. 
Algorithm \ref{algo:cov_scenarios} is only attempted for $q=1$ since it was developed 
solely to match covariance. 
The flat extensions for Lasserre's algorithm is obtained by
a convex relaxation and semidefinite programming (SDP). We start from a rank-$r$ input
moment matrix $\boldsymbol M _{\boldsymbol y}\in
\Rbb ^{m_q\times m_q}$ with the goal to find its flat extension
$\boldsymbol M _{\tilde{\boldsymbol y}}\in \Rbb^{m_{q+l}\times
m_{q+l}}$ for some $l\geq 1$. Given the block matrix representation 
\begin{equation*}
\boldsymbol M _{\tilde{\boldsymbol y}}=\begin{bmatrix}
\boldsymbol A _{\tilde{\boldsymbol y}} &
\boldsymbol B_{\tilde{\boldsymbol y}} \\  
\boldsymbol B_{\tilde{\boldsymbol y}}^{\top} &
\boldsymbol C_{\tilde{\boldsymbol y}}
\end{bmatrix}
\end{equation*}
with
\[
{\boldsymbol A _{\tilde y}\in\Rbb^{m_{q+l-1}\times m_{q+l-1}}},\
                  {\boldsymbol B_{\tilde y}\in\Rbb^{m_{q+l-1}
                  \times (m_{q+l}-m_{q+l-1})}},\
                  {\boldsymbol C_{\tilde y}\in\Rbb^{(m_{q+l}-
                  m_{q+l-1})\times (m_{q+l}-m_{q+l-1})}},
\]
we subsequently minimize the
trace of the bottom-right block. If now flat extension is found,
we increase \(l\) and restart the procedure.

All algorithms have been implemented on a single Intel Xeon E5-2560 
core with 3 GB of RAM, except for Lasserre's algorithm which 
is run on an 18-core Intel i9-10980XE machine with 64 GB RAM. 

\subsection{Multivariate Gaussian mixture distributions}\label{sec:gaussianmixture}
We first examine the different algorithms on simulated data sampled from a Gaussian mixture model with different numbers of dimensions $d$ and having different numbers
of clusters \(c\), components of the mixture distribution, that induce multiple modes
into the joint distribution. 
\begin{figure}
    \centering
    \includegraphics[scale=0.5]{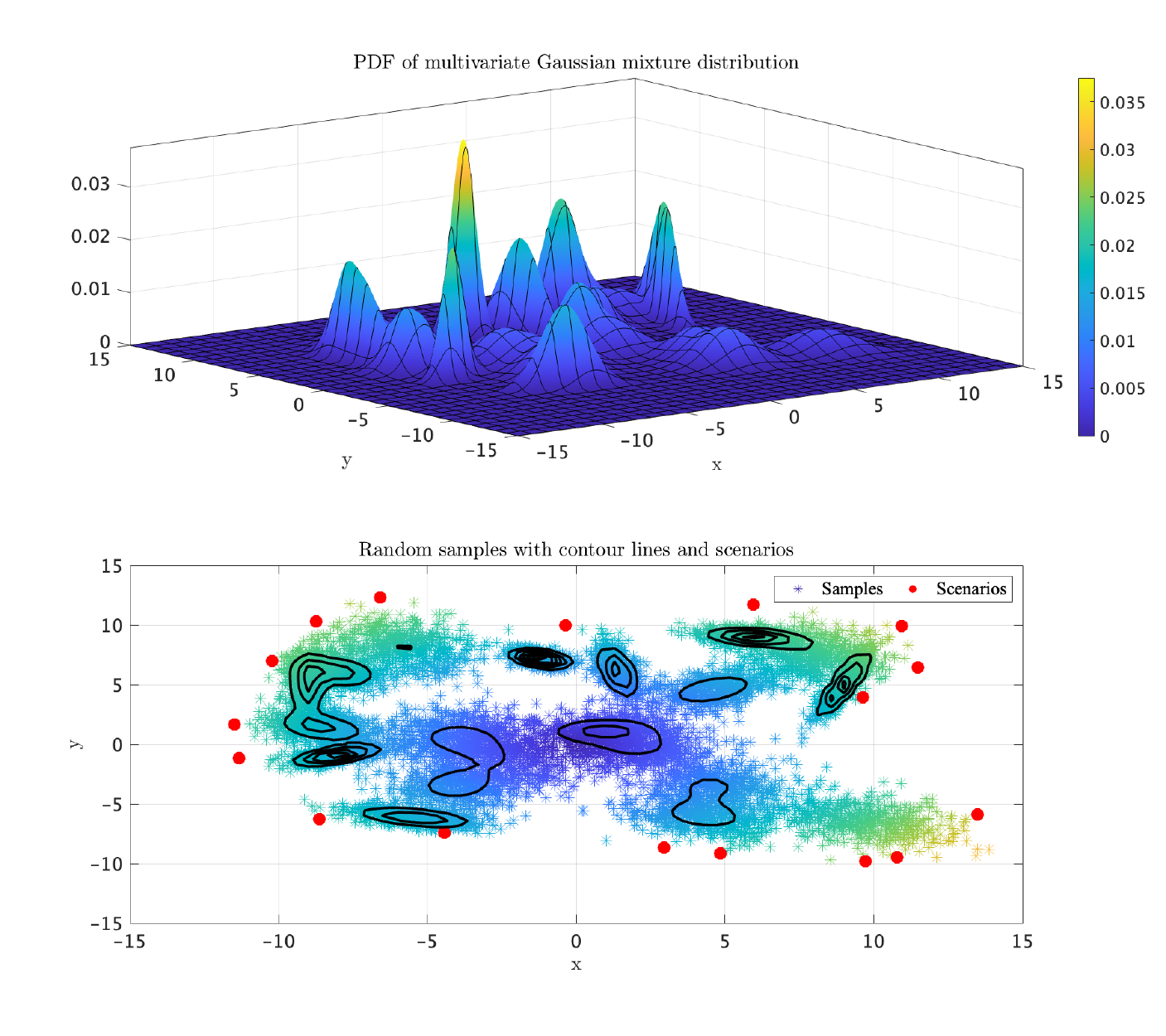}
    \vspace{-2em}
    \caption{\small The OMP scenarios (red dots) from $10000$ random samples of a bivariate Gaussian mixture distribution with their contour lines. The PDF of the distribution is given in the first tile.}
    \label{fig:contour}
\end{figure}
We test for $d=2,5,10,20, 50, 100$ and 
\(c =5,10,20,25,50,100,150,200\). To investigate the efficacy of the algorithms, 
we do not use a higher $c$ for relatively lower $d$ and vice-versa. The means of 
the different clusters are taken to be random vectors of uniformly distributed numbers 
in the interval $(-50,50)$. The variance-covariance matrices for the different clusters 
are either (i) randomly generated positive definite matrices with 
$\mathcal{N}(1,1)$-distributed eigenvalues, or (ii) the identity matrix. Except in the 
case of unit variance-covariance, the different clusters have different variance-covariance 
matrices. The mixing proportions for the different clusters are taken to be either (i) 
random or, (ii) equal. We can categorize them as
\begin{enumerate}
    \item random variance-covariance and random mixing proportion
    \item random variance-covariance and equal mixing proportion
    \item unit variance-covariance and random mixing proportion
    \item unit variance-covariance and equal mixing proportion
\end{enumerate}

For each of the above cases, $100$ data sets are randomly constructed, each containing 
$10000$ samples. For each data set containing the samples, we calculate the sample 
moments $\widehat{\boldsymbol y}$ up to order $2q$ with $q=1,2$ which give rise to 
the empirical moment matrices $\boldsymbol M_{\widehat{\boldsymbol y}}$ of orders 
$1$ and $2$ respectively. The $100$-dimensional and $50$-dimensional data sets are 
computed for $q=1$ only, still resulting in $5151$-dimensional, and $1326$-dimensional 
moment matrices, respectively. The biggest moment matrix for $q=2$ is computed for 
twenty dimensions, resulting in a $10262$-dimensional moment matrix.
It is noteworthy to point out that there is, theoretically speaking, no bound on the 
degree to which we can calculate the moments. However, the computational cost of 
obtaining higher-order moment information is high. We provide a visualization for the case $d=2$ where we first generate $10000$ random samples from a bivariate Gaussian mixture distribution with $20$ clusters, and then retrieve the OMP scenarios as in Figure~\ref{fig:contour}.

We compare the performance of the different algorithms with respect to the relative 
errors, computation times, and the number of scenarios extracted, across different 
dimensions and clusters, as shown subsequently.

\subsubsection{Relative errors}
We compare the behavior of the different algorithms about how well the sample 
moments are matched using the scenarios. For a fair comparison of the relative 
errors for the different algorithms, we define
\begin{equation}\label{eq:rel_error}
    \operatorname{err} \isdef 
    \frac{\big\|\boldsymbol M_{\widehat{\boldsymbol y}} - 
    \boldsymbol V^\top\boldsymbol\Lambda\boldsymbol V\big\|_F}
    {\big\|\boldsymbol M_{\widehat{\boldsymbol y}}\big\|_F}
\end{equation}
where \(\boldsymbol V\) is the Vandermonde matrix of order $q$ generated from the 
samples, and \(\boldsymbol \Lambda\) is the (sparse) diagonal matrix containing 
the respective probability weights of the samples.\\
We plot the relative errors \eqref{eq:rel_error} against the dimensions and number 
of clusters in Figure \ref{fig:comparison_1} and Figure \ref{fig:comparison_2} 
respectively for $q=1$ and $q=2$. In both figures, the $y$-axis represents the 
relative error in the log scale. The $x$-axis contains the different dimensions 
$d$ which is a categorical variable. The color bar depicts the different number 
of clusters and is taken as a categorical variable. The performance of the 
different algorithms is shown in the respective tiles. 
\begin{figure}
    \centering
    \includegraphics[scale=0.58]{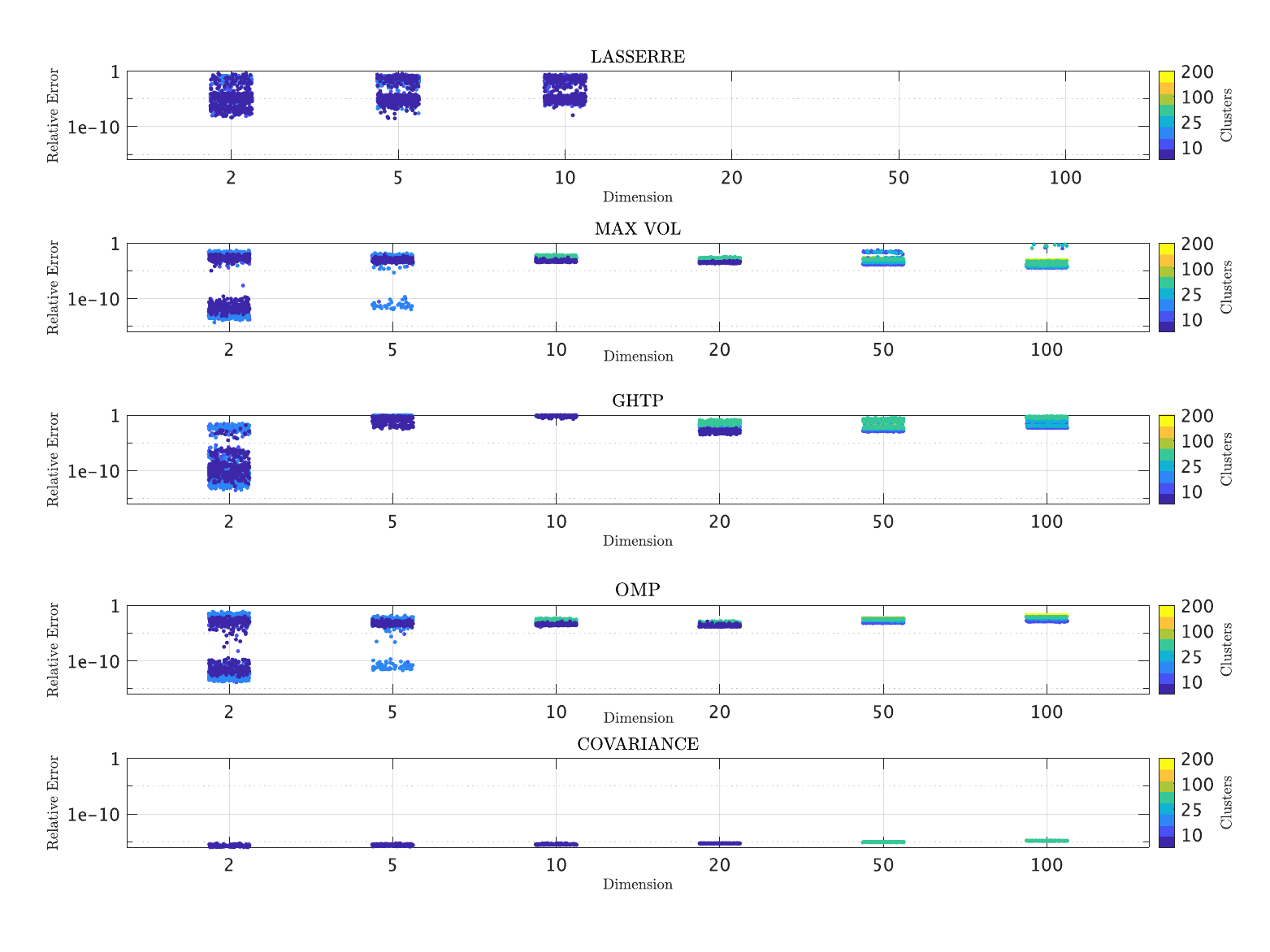}
    \vspace{-2em}
    \caption{\small Relative error comparison for \texorpdfstring{$q=1$}{q=1}.
    The data are generated from Gaussian mixture models according to Section 
    \ref{sec:gaussianmixture}. The relative errors (left $y$-axis) are computed 
    according to \eqref{eq:rel_error}. The $x$-axis shows the dimension $d$ and 
    the right $y$-axis shows the number of clusters entering the mixture distribution.}
    \label{fig:comparison_1}
\end{figure}
\paragraph{Lasserre} The algorithm from \cite{Lasserre} is computationally feasible 
only until dimension $10$ for $q=1$ and $d=5$ for $q=2$. It breaks down thereafter 
since the algorithm involves finding flat extensions requiring the solution of 
large semidefinite relaxations. For each dimension, the relative errors vary 
considerably from the order of $10^{-8}$ to $10^{-1}$ for $q=1$, and $10^{-6}$ 
to $10^{-1}$ for $q=2$. 
There is no particular pattern in their distribution with the number of clusters. 
The results suggest that Lasserre's algorithm is not well suited for large and 
high-dimensional data sets.

\paragraph{Maximum volume} We find that, unlike Lasserre's algorithm, the maximum 
volume algorithm is applicable up to  $d=100$ for $q=1$, and $d=20$ for $q=2$. The 
relative errors range from the order of $10^{-14}$ to $10^{-2}$ for $q=1$, 
and $10^{-4}$ to $10^{-3}$ for $q=2$. They exhibit a slight decrease with the number 
of dimensions, but across dimensions, and a maximum increase by a factor of $10$ with 
the number of clusters. This is a considerable improvement from Lasserre's algorithm 
insofar as the relative errors are more stable across dimensions and the number of clusters. 
However, $d=20$ and $q=2$ exhibit a sharp decrease and range in the order of $10^{-10}$, however.
    
\paragraph{GHTP} This algorithm is easily applicable up to $d=100$ for $q=1$, and $d=20$ for $q=2$. 
The relative errors range from the order of $10^{-13}$ to $10^{0}$ for $q=1$. Except for the case 
of $d=2$, the relative errors mostly range from the order of $10^{-4}$ to $10^{0}$. For $q=2$, 
the relative errors range from the order of $10^{-4}$ to $10^{0}$. In terms of the stability 
of the relative errors, although the GHTP fares better than the maximum volume algorithm, 
it is a deterioration when considering the range of the errors. 

\paragraph{OMP} We find that this algorithm is easily applicable until $d=100$ for $q=1$, 
and $d=20$ for $q=2$. The relative errors range from the order of $10^{-14}$ to $10^{-2}$ 
for $q=1$, and  $10^{-10}$ to $10^{-3}$ for $q=2$, with most in the range of $10^{-4}$ to $10^{-3}$. 
Not only are the errors overall much more stable across the different dimensions, but their orders 
are also comparable to that of the maximum volume algorithm. Except for the case of $d=2$, 
the relative errors mostly range from the order of $10^{-4}$ to $10^{-2}$. While the range 
of the relative errors for the OMP is similar to that of the maximum volume algorithm, 
there is a difference in that we do not observe any significant noticeable pattern in the 
distribution of the relative errors across dimensions. 
    
\paragraph{Covariance scenarios} We find that this algorithm is not only applicable until 
$d=100$, but the relative errors are considerably smaller than for all the above algorithms, 
ranging in the order of $10^{-17}$ to $10^{-15}$.  Hence, as stipulated in the earlier section, 
in the special case of covariance matrices, i.e., when $q=1$, the covariance algorithm performs 
the best since it matches exactly all the sample moments until the second moment. 
By its construction, covariance scenarios are not applicable for $q=2$. 
\begin{figure}
    \centering
     \includegraphics[scale=0.65]{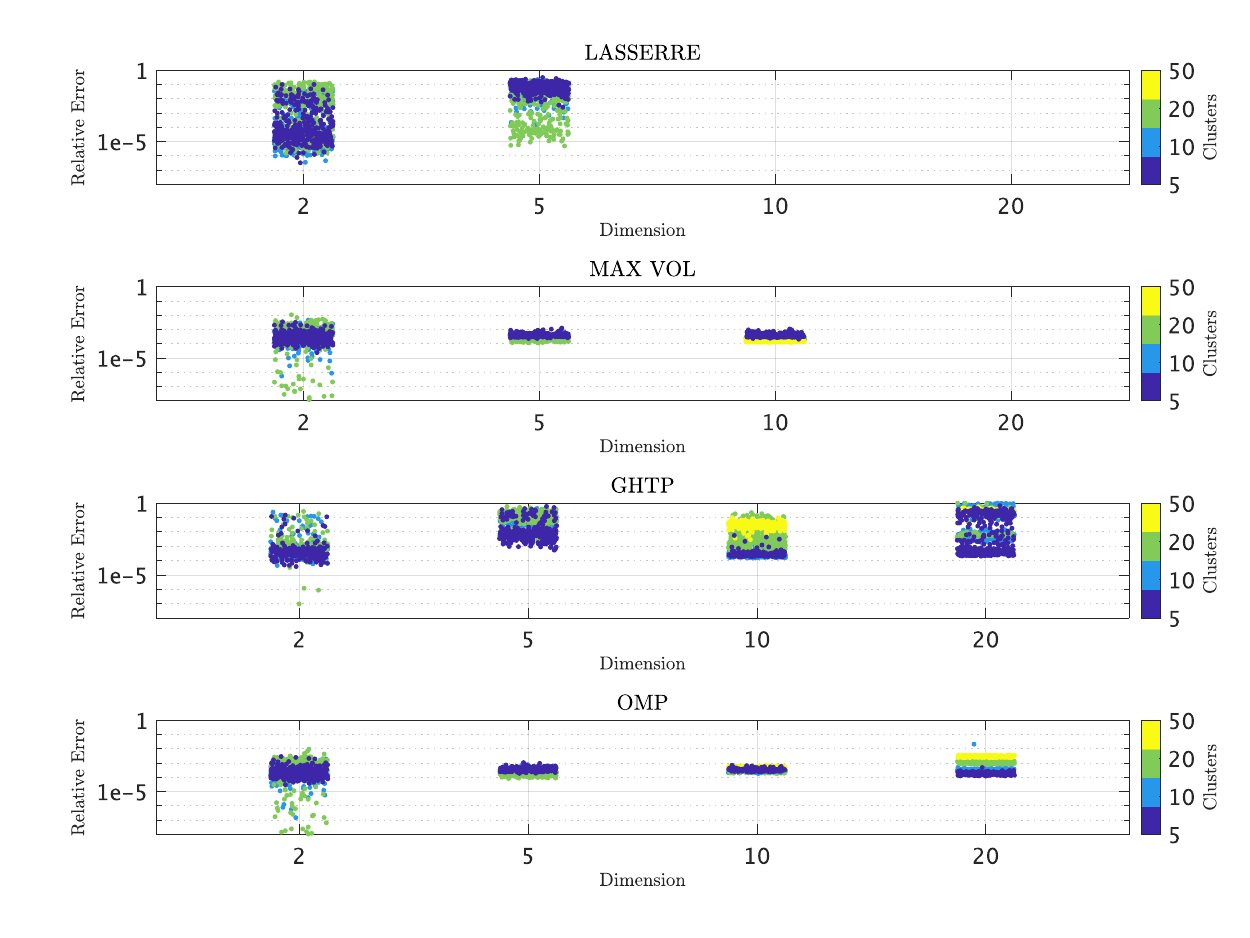}
    \vspace{-1.5em}
    \caption{\small Relative error comparison for \texorpdfstring{$q=2$}{q=2}. The data are 
    generated from Gaussian mixture models according to Section \ref{sec:gaussianmixture}. 
    The relative errors (left $y$-axis) are computed according to \eqref{eq:rel_error}. 
    The $x$-axis shows the dimension $d$ and the right $y$-axis shows the number of clusters 
    entering the mixture distribution.}
    \label{fig:comparison_2}
\end{figure}

\subsubsection{Number of scenarios}
For practical purposes when working with large data sets, given a relative error, sparsity 
is preferred in the number of scenarios extracted when solving the EMP. To that end, we 
compare the number of scenarios extracted from the different algorithms. Except for the 
case of covariance scenarios, that have uniform weights, for the remaining methods, 
we retrieve the weights using the ADMM algorithm described in Section \ref{sec:weightretrieval}. 

For each probability vector, we set the entries less than $10^{-8}$ to zero. As such, we 
consider the number of scenarios to be the number of non-zero entries of the modified 
probability vector. We then plot the number of scenarios against the dimensions and number 
of clusters for $q=1,2$ in Figure \ref{fig:comparison_scen_1} and Figure 
\ref{fig:comparison_scen_2} respectively. In both figures, the $y$-axis denotes the number 
of scenarios and is taken in the log scale. The $x$-axis contains the different dimensions 
which is a categorical variable. The color bar denotes the number of clusters, and it is 
also taken as a categorical variable. The performance of the different algorithms is shown 
in the respective tiles.
\begin{figure}
    \centering
    \includegraphics[scale=0.58]{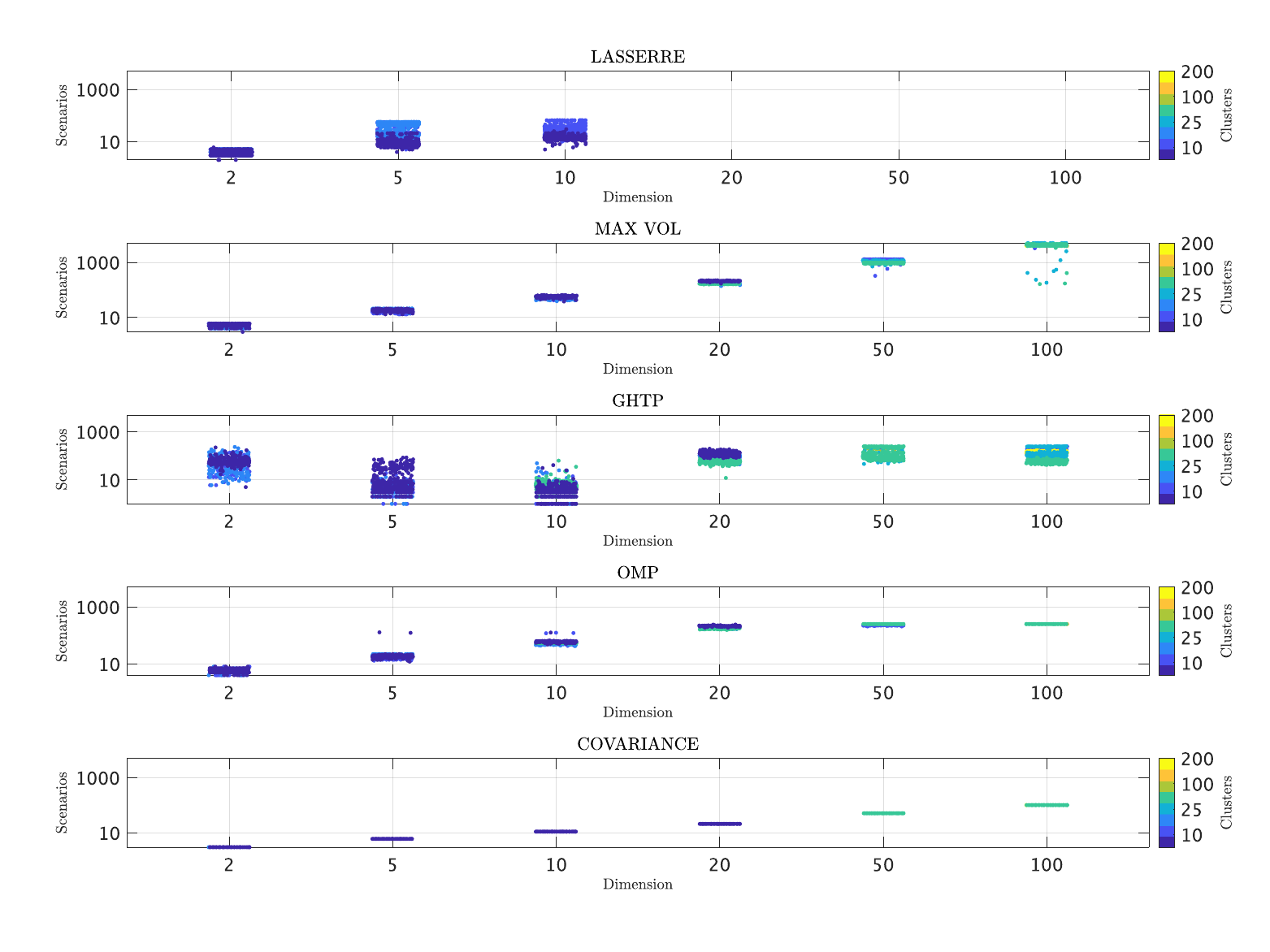}
    \vspace{-2em}
    \caption{\small Number of scenarios comparison for \texorpdfstring{$q=1$}{q=1}. The data 
    are generated from Gaussian mixture models according to Section \ref{sec:gaussianmixture}. 
    The left $y$-axis shows the number of scenarios. The $x$-axis shows the dimension $d$ and 
    the right $y$-axis shows the number of clusters entering the mixture distribution.}
    \label{fig:comparison_scen_1}
\end{figure}
\paragraph{Lasserre} Lasserre's algorithm depends on the flat extension of the moment matrix, 
with the precise rank of the flat extension determining the number of scenarios. Accordingly, 
we find that the number of scenarios, while not too high for each dimension, still shows 
variability for $d=5,10$. For $q=2$, the number of scenarios is also not exceedingly high.
 
\paragraph{Maximum volume} We first observe that the number of scenarios is considerably high 
across the dimensions, ranging from $4000$ to $5000$ for $d=100$. Except for $d=100$, there is no 
considerable difference in the number of scenarios retrieved with the number of clusters. Similar 
to the case for $q=1$, the number of scenarios is considerably high across the dimensions and is 
equal to $10000$ for $d=20$, i.e., it does not give sparse scenarios and considers all the sample 
points. There is also no variation with different clusters for either dimension. Nevertheless, 
it cannot be used for sparse recovery of scenarios from large samples.
 
\paragraph{GHTP} For this algorithm, we find that the number of scenarios is much less compared to 
that of the maximum volume algorithm, and similar to that of Lasserre's algorithm for $d=2,5,10$. 
This result is not surprising, given that the algorithm constructs the index set whose size equals a 
prior input maximum number of iterations (see \cite{GHTP}). For $d=20,50,100$, however, the number of 
scenarios varies considerably with the number of clusters. For $q=2$, the number of scenarios is 
much less, with the maximum being $250$, compared to that of the maximum volume algorithm, nevertheless, 
it varies considerably with the number of clusters for each dimension. 

\paragraph{OMP} The number of scenarios retrieved is somewhat similar to that of the GHTP. This results 
from the fact that we set the maximum number of iterations for the OMP to be the same as that of the GHTP. 
Unlike the GHTP however, we find that in this case, there is no considerable variation in the number of 
scenarios with the number of clusters, for each dimension. Therefore, the OMP outperforms the above 
algorithms, with regard to the number as well as the consistency in the retrieval of scenarios. 
The same interpretation holds for $q=2$. 
 
\paragraph{Covariance scenarios} The range for the number of scenarios is the least among all the 
algorithms for all the dimensions. Furthermore, there is no variation in the case of the different 
clusters for each dimension. This reaffirms that the covariance scenarios perform the best about the 
number of scenarios retrieved, keeping in mind that they are applicable only in the case $q=1$. 
\begin{figure}
    \centering
    \includegraphics[scale=0.65]{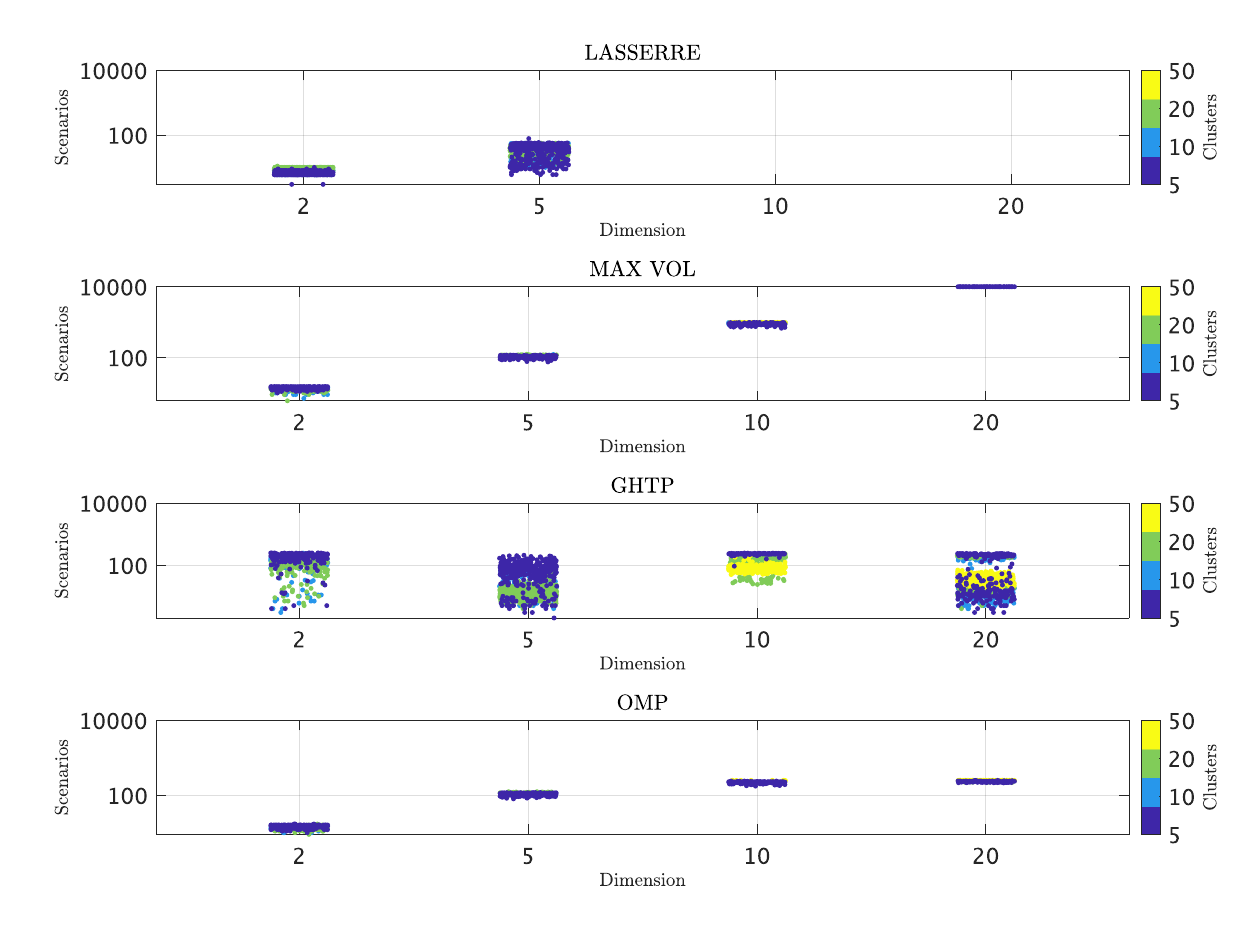}
    \vspace{-1.5em}
    \caption{\small Number of scenarios comparison for \texorpdfstring{$q=2$}{q=2}. The data are 
    generated from Gaussian mixture models according to Section \ref{sec:gaussianmixture}. 
    The left $y$-axis shows the number of scenarios. The $x$-axis shows the dimension $d$ 
    and the right $y$-axis shows the number of clusters entering the mixture distribution.}
    \label{fig:comparison_scen_2}
    \setlength{\belowcaptionskip}{-100pt}
\end{figure}

\subsubsection{Computation time} Finally, we test how fast the algorithms are in computing 
the scenarios from large empirical data sets, keeping practical applications in mind. 
To that end, we plot the CPU run times of the different algorithms against the different 
dimensions and number of clusters for $q=1,2$ in Figure \ref{fig:comparison_time_1} and 
Figure \ref{fig:comparison_time_2} respectively. We take the $y$-axis to be the computation
times in the log scale. The $x$-axis contains the different dimensions, which is a categorical
variable. The color bar depicts the different number of clusters and is also taken as a 
categorical variable.
\begin{figure}
    \centering
    \includegraphics[scale=0.58]{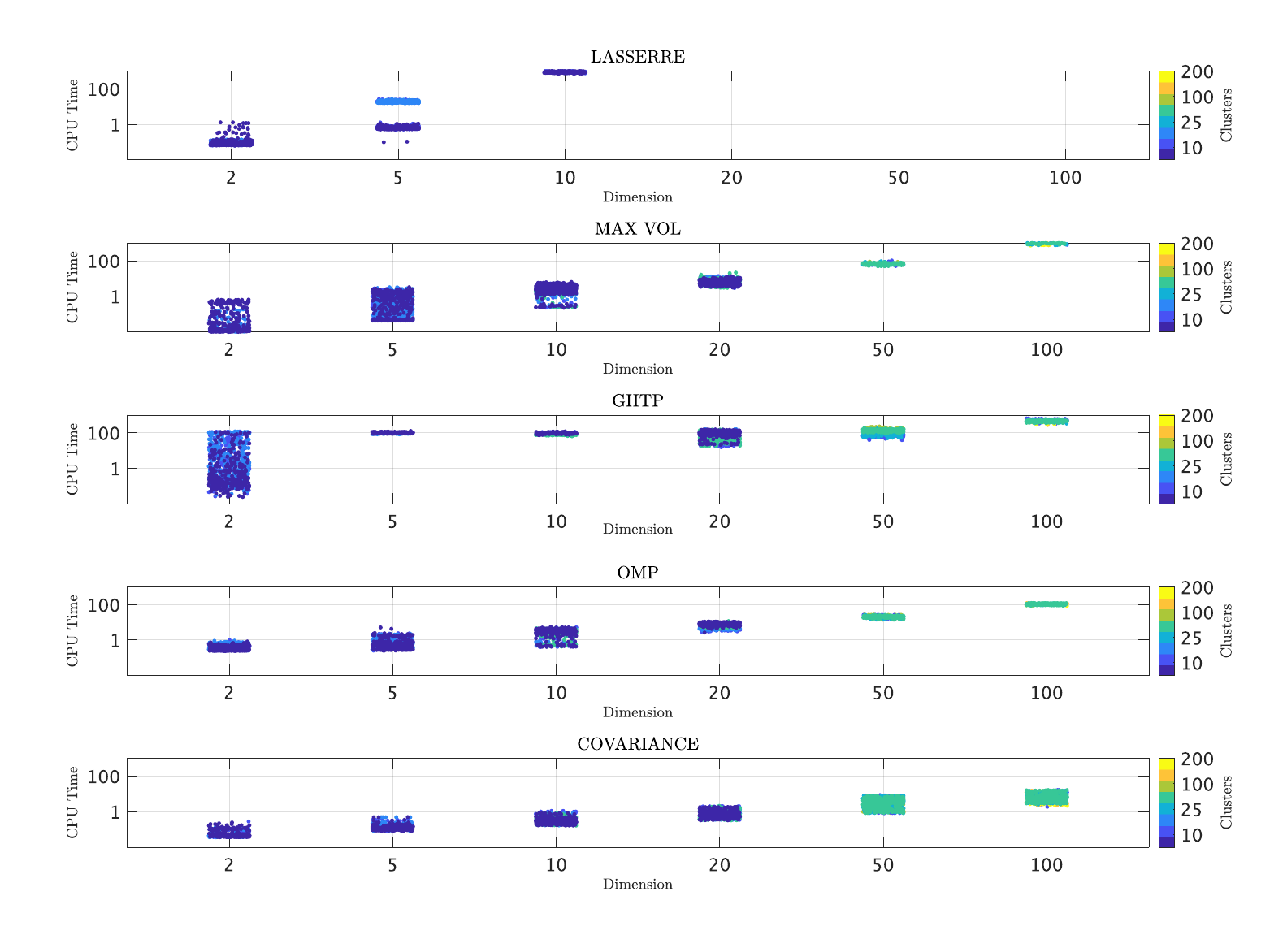}
     \vspace{-2em}
    \caption{\small Computation time comparison for \texorpdfstring{$q=1$}{q=1}. The data 
    are generated from Gaussian mixture models according to Section \ref{sec:gaussianmixture}. The left $y$-axis shows the computation time. The $x$-axis shows the dimension $d$, and the 
    right $y$-axis shows the number of clusters entering the mixture distribution.}
    \label{fig:comparison_time_1}
\end{figure}
\paragraph{Lasserre} Computation times range from $10^{-1}$ to $10^{3}$ seconds just the until 
$d=10$ only. For $d=5$, the run times increase with the number of clusters, by an order of $10$. 
Therefore, Lasserre's algorithm becomes computationally costly in the face of higher dimensions. 
For $q=2$, We find that the times range from $10^{-1}$ to $10^{3}$ seconds until $d=5$ only. 
For $d=5$, the run times increase with the number of clusters, by an order of $100$.
 
\paragraph{Maximum volume} The run times range from $10^{-2}$ to $10^3$ seconds across the different 
dimensions and clusters. A noticeable aspect of this algorithm is that while the run-time increases, 
the variation among them, however, decreases with an increase in the dimension. For $q=2$, 
the run times range from $10^{-1}$ to $10^3$ seconds. Similar to $q=1$, the run times of the 
algorithm for this case increase, and the variation among them, decreases with an increase in 
the dimension. However, a particular drawback remains that the run times show sharp increases 
with the number of dimensions.
    
\paragraph{GHTP} For the GHTP algorithm, the times range from $10^{-1}$ to $10^{2}$ seconds. 
Except for the case of $d=2$, overall the run times of the GHTP algorithm lie in the range of 
$10^{2}$, i.e., it is fairly similar across the different dimensions and clusters and does not 
increase with an increase in either. For $q=2$, the times range from $10^{2}$ to $10^{3}$ seconds. 
Unlike in the case for $q=1$, here, we observe that in general, the times increase with the number 
of dimensions, with a sharp increase for $d=20$.
    
\paragraph{OMP} The run times of the OMP range from $10^{-1}$ to $10^{2}$ seconds across different 
dimensions and clusters. Similar to the maximum volume algorithm, the OMP also shows a gradual 
increase in the computation times with an increase in the dimensions. Moreover, in general, 
it fares better than the GHTP, whose run time is at least higher by a factor of $10$. For $q=2$, 
we observe that the run times range from $10^{-1}$ to $10^{2}$ seconds. Despite the increase 
in the run times with the number of dimensions, nevertheless, we can conclude that it is still 
considerably faster than all of the above algorithms. Note that except for $d=20$ for all 
the other dimensions, it is still
    below $10$ seconds, which further reinforces the efficiency of the algorithm. 
    
\paragraph{Covariance scenarios} We see that the run times for the covariance scenarios 
range from $10^{-2}$ to $10^{1}$ seconds, which is considerably better than that of 
all the above algorithms. 

\begin{figure}
    \centering
    \includegraphics[scale=0.65]{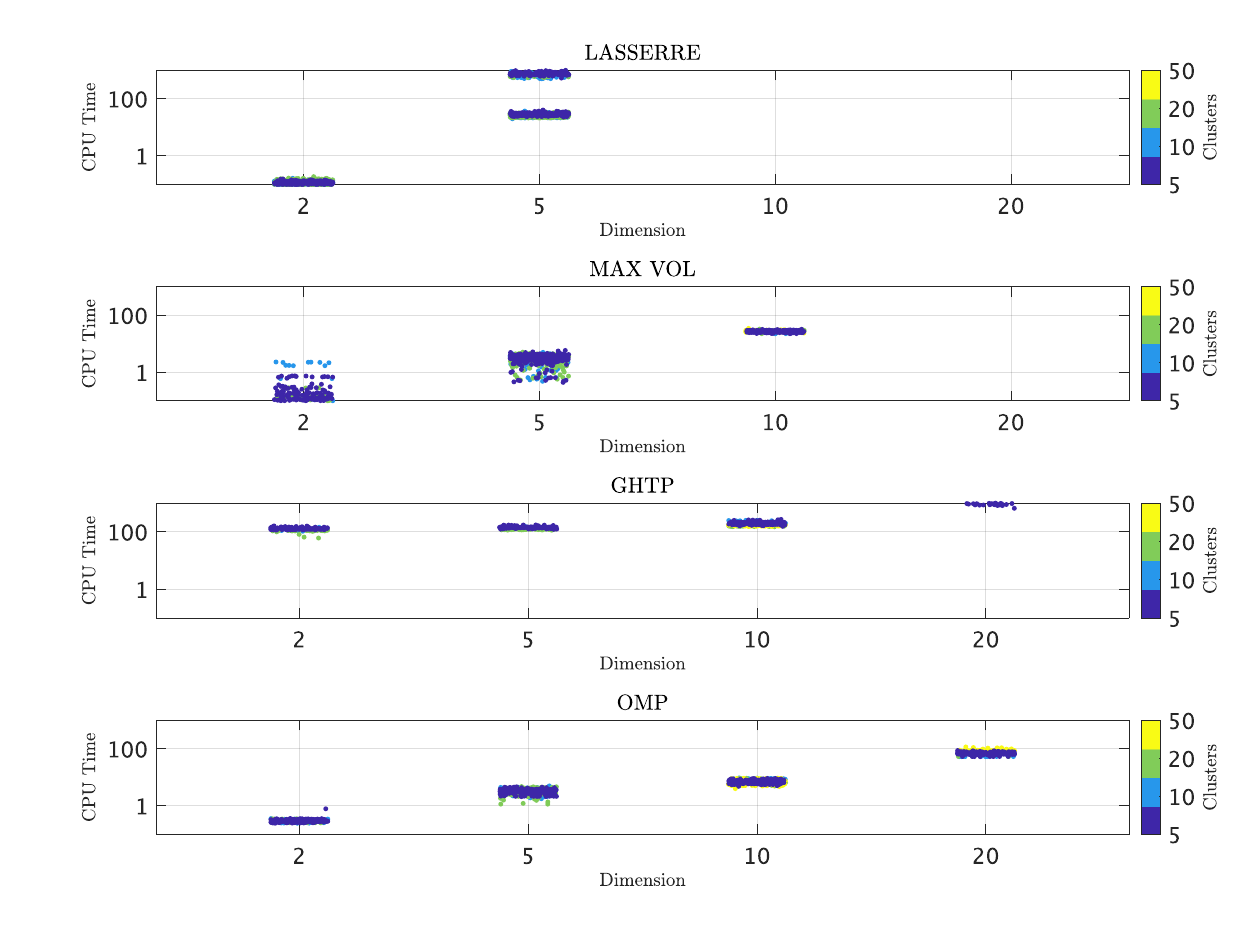}
    \vspace{-1.5em}
    \caption{\small Computation time comparison for \texorpdfstring{$q=2$}{q=2}. 
    The data are generated from Gaussian mixture models according to Section \ref{sec:gaussianmixture}. 
    The left $y$-axis shows the computation time. The $x$-axis shows the dimension $d$, and the 
    right $y$-axis shows the number of clusters entering the mixture distribution.}
    \label{fig:comparison_time_2}
\end{figure}

\subsection{Portfolio optimization with CVaR constraints} 
In this section, we discuss an application of the proposed OMP algorithm in finance using 
excess return data from 
\href{https://mba.tuck.dartmouth.edu/pages/faculty/ken.french/data_library.html}{Fama-French}. 
The data set contains more than $25000$ daily excess returns of $25$ financial assets. A portfolio, or trading strategy,  is defined through the number and choice of assets.
An important aspect of portfolio risk management is to address extreme outcomes, wherein investors are concerned by the multivariate nature of risk and the scale of the portfolios 
under consideration. For decision-making, any sensible strategy involves dimension reduction 
and modeling the key features of the overall risk landscape. To this end, we consider here the 
problem of maximizing expected portfolio returns with expected shortfall constraints as a proxy 
for the entailing risk. Expected shortfall, also known as conditional value-at-risk (CVaR), 
is a coherent risk measure that quantifies the tail risk an investment portfolio has,
see \cite{CVaR}. In 2016, the \href{https://www.bis.org/bcbs/basel3.htm?m=76}{Basel committee} 
 hosted by the Bank of International Settlements (BIS) proposed an internationally agreed set of regulatory measures in response to the financial crisis of 2007-2009; one of them was that the market risk capital of banks should be
calculated with CVaR or expected shortfall, see \cite{BCBS2019}. Nevertheless, methods for portfolio optimization that rely on such a downside risk measure as CVaR, are often difficult to implement, do not scale efficiently, and may result in less-than-ideal portfolio allocations. To resolve these issues and construct an optimal portfolio that minimizes shortfall risk, we consider the following optimization problem:
\begin{equation}\label{eq:portfolio_optimization}
    \begin{split}
        \boldsymbol w^{*} &= \underset{\boldsymbol w \in \mathbb R^d}{\operatorname{argmin}} 
         \quad -\boldsymbol w^\top \boldsymbol \mu \\
         \text{subject to: } &\text{CVaR}_{\alpha} \leq \delta, \quad \boldsymbol w \in W,
    \end{split}
\end{equation}
where \(\boldsymbol \mu\) is the expected return of the individual assets, 
\(\text{CVaR}_{\alpha}\) denotes the conditional value-at-risk,
at the \(\alpha\) level with \(\alpha \in (0,1)\), and \(\delta\) represents the maximum portfolio 
loss that is acceptable.
Specifically, the CVaR or expected shortfall is the expected loss 
conditional on the event that the loss exceeds
the quantile associated with the probability \(\alpha\). Further, $W$ can capture constraints like relative portfolio weights, short-sale restrictions, 
limits on investment in a particular asset, etc. 

We use the reformulation of the CVaR as in \cite{Uryasev2001} to find a global solution. 
While, generally for the optimization approach to work, simulations using a Monte-Carlo approach or bootstrapping of historical data are required, we use our generated scenarios, using the OMP, 
to showcase the applicability in the case of non-smooth optimization as well. The scenarios from the OMP capture the underlying tail risk, which leads to efficient portfolios whose risk lies below a certain threshold.
\begin{figure}
    \centering
    \includegraphics[scale=0.5]{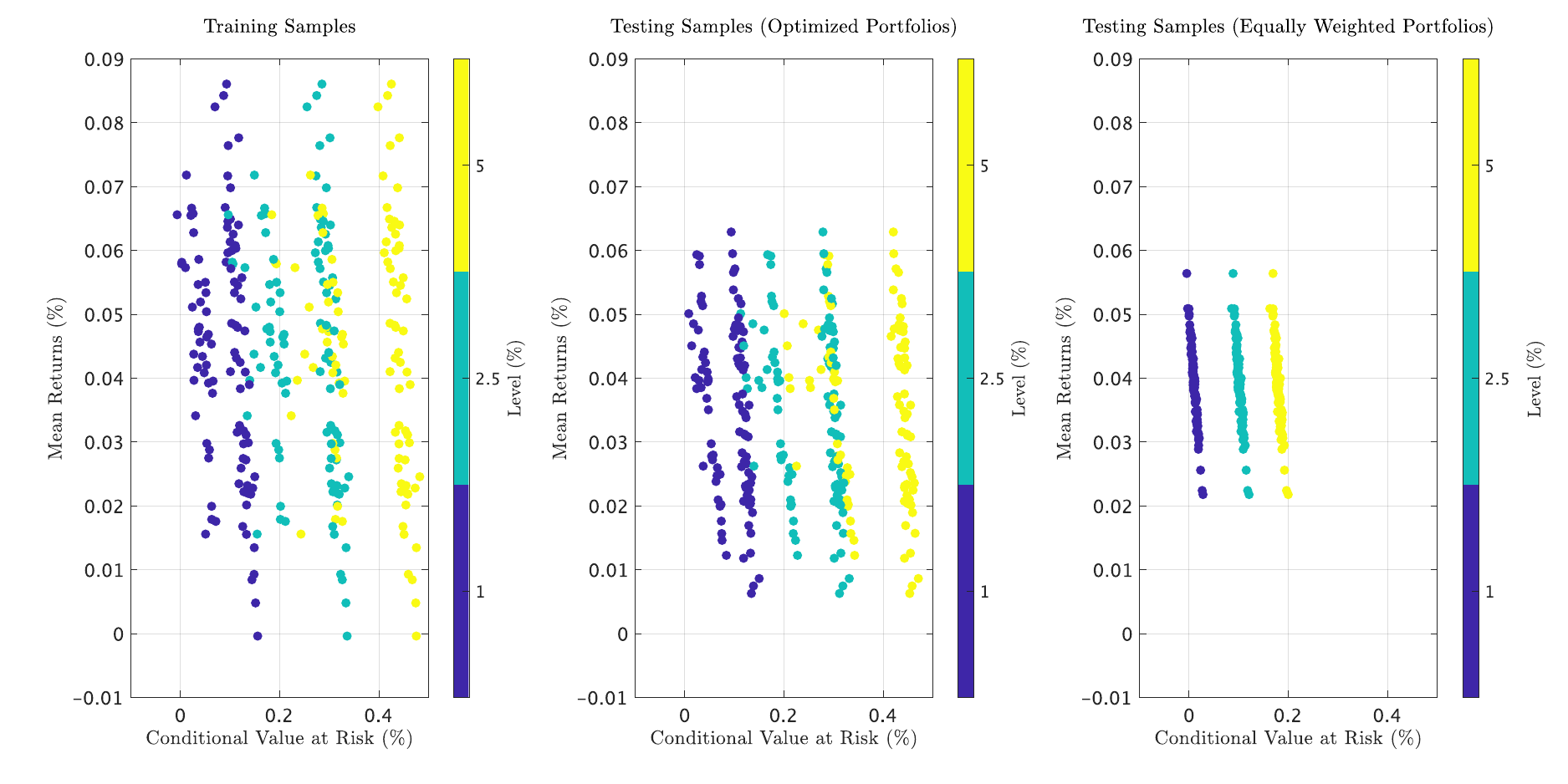}
    \vspace{-1em}
    \caption{\small Daily mean returns versus CVaR
    for the worst $1\%, 2.5\%$ and $5\%$ cases of the expected portfolio loss applied to the training and testing samples with both optimized portfolios and the naive portfolio. The data are daily 
    returns from the $25$-dimensional \href{https://mba.tuck.dartmouth.edu/pages/faculty/ken.french/data_library.html}{Fama-French} portfolios.}
    \label{fig:portfolio_cvar}
\end{figure}
To that end, we split the portfolio data set into a training set of $10000$ observations and a testing 
set of the rest of the observations. We first extract the scenarios and the corresponding probabilities 
from the training data set using the OMP algorithm, see Algorithm \ref{algo:OMP}. We then solve problem \eqref{eq:portfolio_optimization} using the scenarios and the approach suggested in \cite{Uryasev2001}, 
taking $\delta$ to be the conditional value-at-risk for the the naive portfolio rule 
\({\boldsymbol w}=(1/d){\boldsymbol 1}\) at the confidence levels $\alpha = 0.95, 0.975, 0.99$, which correspond to the expected loss in the worst $5\%, 2.5\%$ and $1\%$ cases respectively. The $\operatorname{CVaR}_\alpha$ in the constraint is the corresponding conditional value-at-risk of the expected portfolio loss in the worst $5\%, 2.5\%$ and $1\%$ case as well. We perform back-testing using the optimized portfolio weights 
on the test sample observations. We perform $100$ simulations and observe the distribution of the expected 
daily returns versus the CVaR or the expected shortfall for both the training and testing samples in 
Figure \ref{fig:portfolio_cvar}. Furthermore, we plot the same for the case of the equally weighted portfolios. All the observations are considered as percentages, the color bar 
indicates the $\alpha$ level considered for the optimization problem. The figure shows that the 
optimized portfolios perform well on the testing set as well, achieving a yearly mean return of about 
$12\%$ to $14\%$ in general; hence, our realized scenarios can still capture the behavior of the 
underlying asset returns considerably well, even when trying to match moments of non-smooth 
functions of polynomials. 


\section{Conclusion and future work}\label{sec:Conclusion}
We have proposed two algorithms to find scenarios representing large samples of panel data. 
The first converts estimated sample covariance matrices to a set of uniformly distributed 
scenarios that possibly have not been observed before. The second picks a particular subset 
of realized data points, considering higher-order moment information present in the sample. 

The numerical studies suggest that both algorithms perform well with respect to computational 
efficiency and accuracy relative 
to extant proposals, such as the maximum volume approach by \cite{bittante}, 
greedy hard thresholding, and \cite{Lasserre} multivariate Gaussian quadrature 
in particular in higher dimensions. 

Our framework allows for extensions with non-uniform weighting of the sample points, 
and expectations of a bigger class of functions rather than just polynomials. 
We expect this to be beneficial for non-smooth scenario-based problems as well.